\documentclass[journal]{IEEEtran}
\usepackage{cite,soul,xcolor,amsthm,physics}
\usepackage{color}

\usepackage{framed}
\colorlet{shadecolor}{yellow!75}

\usepackage{graphicx,epsfig,mathptmx,times,amsmath,
amssymb,flushend,gensymb,subfigure,fancyhdr}
\usepackage{lipsum}

\ifCLASSINFOpdf
\fi
\usepackage{epstopdf}

\newtheorem{remark}{Remark}

\newtheorem{assumption}{Assumption}
\newtheorem{theorem}{Theorem }

\begin{document}
\title{Closed-Loop Error Learning Control for Uncertain Nonlinear Systems With Experimental Validation on a Mobile Robot}
\author{Erkan~Kayacan,~\IEEEmembership{Senior Member, IEEE}
\thanks{The information, data, or work presented herein was funded in partby the Advanced Research Projects Agency-Energy (ARPA-E), U.S.Department of Energy, under Award Number DE-AR0000598.}
\thanks{The author is with the School of Mechanical $\&$ Mining Engineering, the University of Queensland, Brisbane. QLD 4072 Australia.
e-mail: e.kayacan@uq.edu.au}

}

\markboth{IEEE/ASME TRANSACTIONS ON MECHATRONICS, VOL. 24, NO. 5, OCTOBER 2019, PREPRINT}
{Shell \MakeLowercase{\textit{et al.}}: Bae Demo of IEEEtran.cls for Journals}
\maketitle

\begin{abstract}

This paper develops a Closed-Loop Error Learning Control (CLELC) algorithm for feedback linearizable systems with experimental validation on a mobile robot. Traditional feedback and feedforward controllers are designed based on the nominal model by using Feedback Linearization Control (FLC) method. Then, an intelligent controller is designed based on sliding mode learning algorithm that utilizes closed-loop error dynamics to learn the system behavior. The controllers are working in parallel, and the intelligent controller can gradually replace the feedback controller from the control of the system. In addition to the stability of the sliding mode learning algorithm, the closed-loop stability of an $n$th order feedback linearizable system is proven. The simulation results demonstrate that CLELC algorithm can improve control performance (e.g., smaller rise time, settling time and overshoot) in the absence of uncertainties, and also provides robust control performance in the presence of uncertainties as compared to traditional FLC method. To test the efficiency and efficacy of CLELC algorithm, the trajectory tracking problem of a tracked mobile robot is studied in real-time. The experimental results demonstrate that CLELC algorithm ensures high-accurate trajectory tracking performance than traditional FLC method.

\end{abstract}


\IEEEpeerreviewmaketitle

\section{Introduction}

\IEEEPARstart{C}{ontrol}  of uncertain nonlinear systems is one of the most important topics in modern control engineering \cite{KAYACAN2012863, Chen2016, KAYACAN2017276, MATRAJI2018167, 7879290, 8306927, 8316974}. Controllers intend to achieve the best control performance in the presence of the worst uncertainties in robust control approaches, and a high controller gain is the general method to handle the effect of uncertainties in nonlinear control theory \cite{Kayacan2017Hinf, 7323849}. However, these methods result in substantial control action, and very powerful actuators are thus demanded to perform unnecessarily excessive control actions \cite{7762162}. Furthermore, the robust control performance is generally established by sacrificing the nominal control performance because the nominal control performance is not taken into account \cite{7574310, GUENOUNE201723}. Therefore, a control method is required to remain or improve nominal control performance in the absence of uncertainties, and exhibit robust control performance in the presence of uncertainties \cite{erkanasjc, Huang2015, AMINI20171, 8170236, 8326557}.

Autonomous guidance, navigation and control of mobile robots in unstructured, off-road is one of the major problems in field robotics \cite{Mousazadeh2013, Deremetz2017, BEGNINI201727}. Mobile robots in off-road terrain face to different surface materials and terrain topographies; therefore, learning-based control techniques are necessary for mobile robots to learn the unmodeled surface conditions (e.g., grass, sand, and snow) and complex robot dynamics \cite{Ostafew2016}. Modeling of these effects is so arduous because it is not always possible to know the terrain type and wheel-terrain interactions. Even if models exist, the identification of model parameters is cumbersome. Therefore, researchers have focused on learning-based controllers for mobile robots \cite{Rossomando2014, yu_chen_2015, Ostafew2013}. Online adaptive controller based on a kinematic model with wheel slip angles was developed in \cite{Cariou2009}, where observers estimated the slip angles. Modeled and identified uncertainties partially improved mobile robot path tracking performance; however, it was not sufficient. To model and identify uncertainties fully, a framework consisting of moving horizon estimation and model predictive control methods was proposed for articulated unmanned ground vehicles \cite{7525615, Kayacan2018bc, 8409989} and mobile robots \cite{erkanjfr, KayacanRSS}. In these works, an adaptive kinematic model was derived by adding traction parameters for longitudinal and side slips into the traditional kinematic model. Nonlinear moving horizon estimation method estimated these traction parameters and nonlinear model predictive controller designed based on the adaptive kinematic model with these traction parameters estimates generated control signals. Since parameter estimation in nonlinear moving horizon estimator is carried out in the arrival cost, the estimates are assumed to be changing very slowly, and the most recent estimates do not have any information on the measurements in the estimation horizon. Therefore, our learning approach must (i) take the most recent measurements into account, (ii) be robust against fast changing working conditions and (iii) enable the representation of complex disturbance characteristics.

As distinct from previous studies on sliding mode learning algorithms \cite{erkansmlc}, the Closed-Loop Error Learning Control (CLELC) algorithm developed in this paper does not require a high controller gain in feedback control action to guarantee system stability so that it does not cause chattering effects on control signals.  Moreover, the CLELC algorithm demonstrates two remarkable features. First, it does not cause any adverse effects on the system in the absence of uncertainties, and it additionally improves control performance when compared to nominal control performance obtained by Feedback Linearization Control (FLC). Second, the proposed method exhibits robust control performance for feedback linearizable systems with uncertainties.

The main contributions of this study are as follows: A novel control method termed as a CLELC algorithm is developed and implemented in this paper. In this novel control scheme, closed-loop error dynamics are utilized to learn system behavior for adaptation. After a short duration, the intelligent controller takes the overall control action while the output of the feedback controller converges to zero. Moreover, the overall system stability for a feedback linearizable system with uncertainties is proven based on Lyapunov stability theory under the bounded error, reference and feedforward action rates. The stability analysis shows that an $n$th order feedback linearizable system controlled by the CLELC algorithm is stable. Along with the theoretical results, this study demonstrates trajectory tracking-test results of the developed CLELC algorithm on a tracked mobile robot. 

The organization of the paper is as follows: The problem formulation is given in Section \ref{sec_prob_form}. The CLELC algorithm is formulated Section \ref{sec_cls}. The intelligent controller is designed and the basics of the sliding mode learning algorithm are given in Sections \ref{sec_anfs} and \ref{sec_smla}, respectively, while the closed-loop system stability is proven in Section \ref{sec_overalstability}. Simulation results for a third-order nonlinear system and experimental validation on a mobile robot are given in Sections \ref{sec_simresults} and \ref{sec_expval}, respectively. Finally, conclusions of  the study are summarized in Section \ref{sec_conc}.

\section{Problem Formulation}\label{sec_prob_form}

An $n$th order feedback linearizable system with uncertainties is written as below:
\begin{eqnarray}\label{eq_nonlinearsystem}
\dot{x}_{i} &=& x_{i+1} \nonumber \\
\dot{x}_{n} &=&  a(\textbf{x}) +  b(\textbf{x}) u + \Delta (\textbf{x},u) 
\end{eqnarray}
where $i=1,2,\dots,n-1)$, $\textbf{x}=[x_{1}, x_{2}, \hdots, x_{n}] \in R^{n}$ is the state vector, $u \in R$ is the control input, $a(\textbf{x})$ and $b(\textbf{x})$ are smooth nonlinear functions, and $\Delta (\textbf{x}, u)$ is the term for system uncertainty where $\Delta (\textbf{x}, u) = \Delta a(\textbf{x}) + \Delta b(\textbf{x}) u$.

Traditional feedback linearization control (FLC) law is formulated below:
\begin{equation} \label{eq_ndi_controllaw}
u = b^{-1}(\textbf{x}) \Big( -a(\textbf{x}) + u_{b} + u_{f} \Big)
\end{equation}
where $u_{b}$ and $u_{f}$ are the feedback and feedforward control inputs, and formulated below:
\begin{eqnarray}\label{eq_fb}
u_{b} & = & \textbf{k} \textbf{e} = k_{1} e_{1} + k_{2} e_{2} + \hdots + k_{n} e_{n} \\\label{eq_ff}
u_{f} & = & \dot{r}_{n} 
\end{eqnarray}
where $\textbf{k}=[k_{1}, k_{2}, \hdots, k_{n}]$ is the feedback controller gain vector and positive, i.e., $k_{i} (i=1,2,\dots,n)>0$, and $\textbf{e}=[e_{1}, e_{2}, \hdots, e_{n}]^{T}=[r_{1}-x_{1}, r_{2}-x_{2}, \hdots, r_{n}-x_{n}]^{T}$ is the error vector where $\textbf{r}=[r_{1}, r_{2}, \hdots, r_{n}]^{T}$ is the reference vector and defined as $[\dot{r}_{1}, \dot{r}_{2}, \hdots, \dot{r}_{n}]^{T}=[r_{2}, r_{3}, \hdots, r_{n+1}]^{T}$ so that the reference $r_1$ must be continuously $(n+1)$ times differentiable.

If FLC law \eqref{eq_ndi_controllaw} is employed to control an nth-order feedback linearizable system with uncertainties \eqref{eq_nonlinearsystem}, the closed-loop error dynamics are calculated below:
\begin{equation}\label{eq_ndi_error}
\dot{e}_{n} + k_{n} {e}_{n} + k_{n-1} e_{n-1} + \hdots +  k_{1} e_{1} = - \Delta (\textbf{x},u)
\end{equation}

Equation \eqref{eq_ndi_error} shows that the error dynamics cannot be driven to zero under FLC control law \eqref{eq_ndi_controllaw} in the presence of uncertainties.

\section{Closed-Loop Error Learning Control Algorithm}\label{sec_cls}

In this study, a CLELC algorithm is developed where the feedback, feedforward and intelligent controllers are working in parallel as illustrated in Fig. \ref{fig_CEL}. The total control law for a CLELC algorithm is formulated as follows:
\begin{equation} \label{eq_new_controllaw}
u = b^{-1}(\textbf{x}) \Big( -a(\textbf{x}) + u_{b} + u_{f} + u_{n}  \Big)
\end{equation}
where $u_{b}$ and $u_{f}$ are respectively the traditional feedback and feedforward control inputs, and $u_{n}$ is the output of the intelligent controller. The total control input is formulated as follows:
\begin{equation} \label{eq_totalcontrollaw}
u_{t} = u_{b} + u_{f} + u_{n} 
\end{equation}
The details of the intelligent controller consisting of an adaptive neuro-fuzzy structure and a sliding mode learning algorithm are given in the next subsections.
\begin{figure}[t!]
  \centering
  \includegraphics[width=1.0\columnwidth]{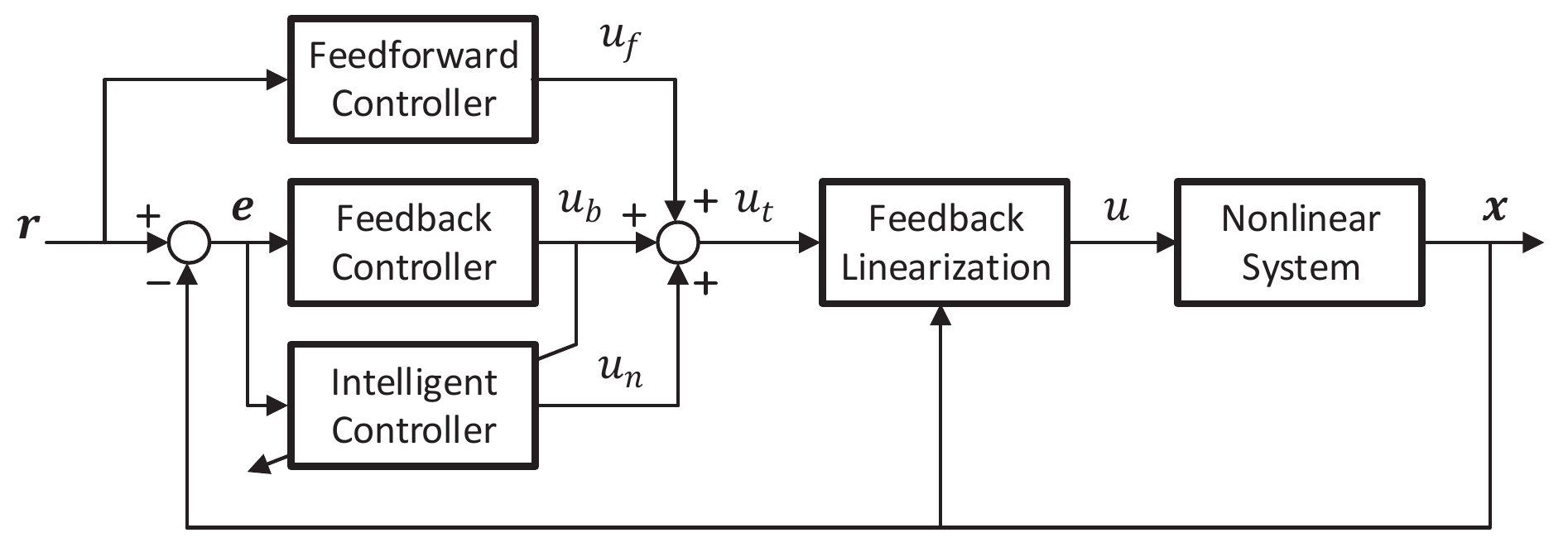}
  \caption{Closed-loop error learning control scheme.}\label{fig_CEL}
\end{figure}

\subsection{Adaptive Neuro-Fuzzy Structure}\label{sec_anfs}

The fuzzy \emph{if-then} rule of a zeroth-order Takagi-Sugeno-Kang model with $I$ input variables where the consequent part is a crisp number can be defined as follows:
\begin{equation}
R_{r}: \;\; \textrm{If} \; e_1 \; \textrm{is} \; A_{1k} \; \hdots \; \textrm{and} \; e_i \; \textrm{is} \; A_{ik}, \;\; \textrm{then} \;\; f_{r}=d_{r}
\end{equation}
where $e_{i} (i=1,\hdots, I)$ are the inputs of the Takagi-Sugeno-Kang model, $A_{ik}$ is the $k$th fuzzy membership function $(k=1,\hdots,K)$ corresponding to the input $i$th input variable. $K$ is the number of membership functions for the $i$th input. The consequent part of the rules $f_{r} (r=1,\dots,N)$ is the output function.

In adaptive neuro-fuzzy structure, the crisp inputs $e_i$ are mapped into fuzzified values using membership functions \cite{AYAS201744}. The membership values for the inputs  $\mu_{ik}(e_i)$ are defined by using Gaussian membership functions as below:
\begin{equation}\label{eq_mu}
\mu_{ik}(e_i) =\exp \Bigg(- \frac{1}{2}\Big(\frac{e_i - c_{ik}}{\sigma _{ik}} \Big)^2 \Bigg)  
\end{equation}
where $c_{ik}$ and $\sigma_{ik}$ are respectively the mean and standard deviation of the membership functions to be adjusted.

The firing strengths of the rules $w_{r}$ are calculated by using prod-t norm operator as follows:
\begin{equation}\label{eq_wr}
 w_{r} = \mu_{A1}(e_1) \ast \mu_{A2}(e_2) \ast \hdots  \ast \mu_{AI}(e_I)
\end{equation}
The prod-t norm operator is preferred due to the simplicity for the stability analysis and the normalized firing strengths of the rules are computed as below:

\begin{equation}\label{eq_wrwidetilde}
 \widetilde{w}_{r} = \frac{w_{r}}{\sum\limits_{r=1}^{N}{w}_{r}} 
\end{equation}

The output signal of the neuro-fuzzy structure $u_n$ is computed as follows:
\begin{equation}\label{eq_un}
u_n(t) = \sum\limits_{r=1}^{N} \widetilde{w}_{r}f_{r} 
\end{equation}

The following vectors can be specified: $\widetilde{\textbf{w}}=\left[\widetilde{w_{1}} \; \widetilde{w_{2}} \; ...\; \widetilde{w_{N}} \right]^{T} $, and $\textbf{f}=\left[ f_{1} \; f_{2} \; ...\; f_{N} \right] $.

\subsection{Sliding Mode Learning Algorithm}\label{sec_smla}

A novel sliding mode learning algorithm for the adaptations of the neuro-fuzzy structure is designed to ensure the system stability in this paper. In contradistinction to previous studies, the closed-loop error dynamics of a feedback linearizable system are defined as the sliding surface for the sliding mode learning algorithm as below:
\begin{eqnarray}\label{eq_ss}
s &=& \Big(\frac{d}{dt}  + \lambda \Big)^n e_{1} \nonumber  \\
&=& \dot{e}_{n} + \frac{n \lambda}{1!} e_{n} + \frac{n (n-1)} {{2!}} \lambda^{2} e_{n-1} + \hdots + \lambda^{n} e_{1} 
\end{eqnarray}
where $n$ is the order of a feedback linearizable system and $\lambda$ is the slope of the sliding surface. The condition $s=0$ guarantees that the error is zero when the system is on the sliding surface, and the $n$th order system is stable. 

The sliding surface is rewritten by taking the feedback control law \eqref{eq_fb} into consideration as follows:
\begin{equation}\label{eq_ss2}
s = \dot{e}_{n} + u_{b} 
\end{equation} 
where the controller gain vector in the feedback control law is defined as 
\begin{equation}\label{eq_ks}
\textbf{k}= [k_{1}, \hdots, k_{n-1}, k_{n}]=\Big[\lambda^{n}, \hdots,  \frac{n (n-1)} {{2!}} \lambda^{2}, \frac{n \lambda}{1!} \Big]
\end{equation}

\begin{assumption}\label{ass_1}
The total input rate  and the second time-derivative of state $x_{n}$ are assumed to be bounded:
\begin{equation}\label{}
\mid \dot{u}_{t} \mid < B_{\dot{u}_{t}} , \quad \textrm{and}  \quad  \mid \ddot{x}_{n} \mid < B_{\ddot{x}_{n}}
\end{equation}
where $B_{\dot{u}_{t}} $ and $B_{\ddot{x}_{n}} $ are considered as positive constants. Moreover, the learning rate $\alpha$ is assumed to be selected larger than the sum of these positive constants
\begin{equation}\label{}
\alpha > B_{\dot{u}_{t}} +  B_{\ddot{x}_{n}} 
\end{equation}
\end{assumption}

In the light of Assumption \ref{ass_1}, since the second time-derivative of $x_{n}$ is bounded, the second time-derivative of $r_{n}$ must also be bounded. 

\begin{theorem}\label{theorem1}
The parameters adaptation laws of the neuro-fuzzy structure are given by the following equations:
\begin{equation} \label{eq_c_ik}
\dot{c}_{ik} = \dot{e}_{i} + (e_{i} - c_{ik}) \alpha \textrm{sgn}\left( s \right)
\end{equation}
\begin{equation}\label{eq_sigma_ik}
\dot{\sigma}_{ik} = - \bigg( \sigma_{ik} + \frac{ (\sigma_{ik} )^3}{(e_{i} - c_{ik})^2} \bigg) \alpha \textrm{sgn}\left( s \right)
\end{equation}

\begin{equation} \label{eq_fr}
\dot{f}_{r} =\frac{\widetilde{\textbf{w}}_{r} }{\widetilde{\textbf{w}}^{T}_{r} \widetilde{\textbf{w}} _{r}} \alpha \textrm{sgn}\left( s \right)
\end{equation}
where $\alpha$ is a sufficiently positive constant as stated in Assumption \ref{ass_1}. This ensures that (i) the feedback control action approaches to zero for a given arbitrary initial condition $u_{b}(0)$ and (ii) $\dot{x}_{n}$ can track $\dot{r}_{n}$ so that $\dot{e}_{n}$ approaches to zero for a given arbitrary initial condition $\dot{e}_{n}(0)$.
\end{theorem}

\begin{proof}
The time derivative of Gaussian membership function \eqref{eq_mu} is written as follows:
\begin{equation}\label{eq_dotmu}
\dot{\mu}_{ik}(e_i) = -M_{ik}\dot{M}_{ik} \mu_{ik}(e_i)  
 \end{equation}
where
\begin{eqnarray}\label{eq_A}
M_{ik}=\frac{e_i - c_{ik}}{\sigma_{ik}}, \quad 
\dot{M}_{ik}=\frac{(\dot{e}_i - \dot{c}_{ik})\sigma_{ik} -(e_i - c_{ik})\dot{\sigma}_{ik}}{\sigma_{ik} ^{2}}
\end{eqnarray}

Combination of \eqref{eq_c_ik}, \eqref{eq_sigma_ik}, and \eqref{eq_A} gives:
\begin{equation}\label{eq_Kr}
K_{r} = \sum\limits_{i=1}^{I} M_{ik}\dot{M}_{ik}=I \alpha \textrm{sgn}\left( s \right)
\end{equation}

The time derivative of \eqref{eq_wr} is written as follows:
\begin{equation}\label{eq_wrdot}
\dot{w}_{r} = - K_{r} w_{r} 
\end{equation}

The time derivative of \eqref{eq_wrwidetilde} is written considering \eqref{eq_wrdot} as follows:
\begin{equation}\label{eq_dotwr}
\dot{\widetilde{w}}_{r} = - \widetilde{w}_{r} K_{r} +\widetilde{w}_{r} \sum\limits_{r=1}^{N}\widetilde{w}_{r} K_{r}
\end{equation}

The stability of the sliding mode learning algorithm is investigated by the following Lyapunov function below:
\begin{equation}\label{eq_V}
V=\frac{1}{2} \left(u _{b} + \dot{e}_{n} \right)^{2}
\end{equation}

The time derivative of the proposed Lyapunov function \eqref{eq_V} is computed below:
\begin{equation}\label{eq_dotV}
\dot{V}= ( u _{b} + \dot{e}_{n}  ) ( \dot{u} _{b} + \ddot{e}_{n}  )\nonumber \\
\end{equation}
Taking \eqref{eq_totalcontrollaw} into consideration, $\dot{u}_{b}$ is inserted into the equation above. Then, it is re-organized consideration \eqref{eq_ss2} as follows:
\begin{equation}\label{eq_dotVV}
\dot{V}=  s (\dot{u}_{t} - \dot{u}_{f} -\dot{u} _{n} + \ddot{e}_{n}  ) 
\end{equation}

The time derivative of the output of the neuro-fuzzy structure \eqref{eq_un} is calculated as follows:
\begin{equation}\label{eq_dotun}
\dot{u} _{n}= \sum\limits_{r=1}^{N} \dot{f}_{r} \widetilde{w}_{r} + f_{r} \dot{\widetilde{w}}_{r}
\end{equation}
If \eqref{eq_dotwr} is inserted into the equation above, \eqref{eq_dotun} can be obtained:
\begin{equation}\label{eq_dotun2}
\dot{u} _{n} = \sum\limits_{i=1}^{N} \bigg\{\dot{f}_{r} \widetilde{w}_{r} + f_{r} \Big(- \widetilde{w}_{r} K_{r} + \widetilde{w}_{r} \sum\limits_{i=1}^{N} \widetilde{w}_{r} K_{r} \Big) \bigg\}
\end{equation}
If \eqref{eq_Kr} is inserted into the aforementioned equation, \eqref{eq_dotun2} can be obtained:
\begin{equation}\label{eq_dotun3}
\dot{u} _{n}  = \sum\limits_{r=1}^{N} \bigg\{ \dot{f}_{r} \widetilde{w}_{r} + f_{r} I \alpha \textrm{sgn}\left( s \right)  \Big(- \widetilde{w}_{r}  +
 \widetilde{w}_{r} \sum\limits_{r=1}^{N} \widetilde{w}_{r} \Big) \bigg\}
\end{equation}

The aforementioned equation is rewritten by taking the definition $ \sum\limits_{i=1}^{N}\widetilde{w}_{r} = 1$ into account as follows:
\begin{equation}\label{eq_dotun4}
\dot{u} _{n} = \sum\limits_{r=1}^{N} \dot{f}_{r} \widetilde{w}_{r}
\end{equation}

By using \eqref{eq_fr},  \eqref{eq_dotun4} is obtained as follows:
\begin{equation}\label{eq_dotun5}
\dot{u} _{n} = \alpha \textrm{sgn}\left( s \right)
\end{equation}
Substitution of \eqref{eq_dotun5} into \eqref{eq_dotV} gives:
\begin{equation}\label{eq_dotV2}
\dot{V} = s \Big[\dot{u}_{t} - \dot{u}_{f} -  \alpha \textrm{sgn}\left( s \right) + \ddot{e}_{n} \Big] 
  \end{equation}
If the aforementioned equation is re-arranged considering $\ddot{e}_{n}=\ddot{r}_{n}-\ddot{x}_{n}=\dot{u}_{f}-\ddot{x}_{n}$
\begin{equation}\label{eq_dotV3}
\dot{V} = s \Big[\dot{u}_{t} -  \alpha \textrm{sgn}\left( s \right) -\ddot{x}_{n} \Big]  
  \end{equation}
As remarked in Assumption \ref{ass_1}, if it is assumed that the total input rate $\dot{u}_{t}$ and the second time-derivative of state $\ddot{x}_{n}$ are upper bounded, it is obtained as follows:
\begin{equation}\label{eq_dotV5}
\dot{V} \leq \mid s  \mid  (B_{\dot{u}_{t}} + B_{\ddot{x}_{n}} - \alpha  )
  \end{equation}
If the learning rate $\alpha$ is large enough, i.e., $\alpha \geq  B_{\dot{u}_{t}} +  B_{\ddot{x}_{n}}$, as stated in Theorem \ref{theorem1}, the time derivative of the Lyapunov function yields
\begin{equation}\label{eq_dotV6}
\dot{V} <  0
  \end{equation}
So that the sliding mode learning algorithm is stable, the feedback control action $u_{b}$ approaches to zero and $\dot{x}_{n}$ can track $\dot{r}_{n}$. This implies that the intelligent controller is capable of learning the system behavior.
\end{proof}

\begin{remark}
Sliding mode learning algorithm learns the system behaviour very fast so that selection of initial conditions for the adaptation laws \eqref{eq_c_ik}-\eqref{eq_fr} does not play a very critical role in system performance. However, if the initial conditions for the adaptation laws are very large, then parameter boosting might occur and cause instability of the training algorithm.
\end{remark}

\subsection{System Stability Analysis}\label{sec_overalstability}

\begin{assumption}\label{ass_dotDelta}
The rate of the uncertainty term in the system dynamics \eqref{eq_nonlinearsystem} is assumed to be bounded:
\begin{equation}\label{eq_bounds}
\mid \dot{\Delta}(\textbf{x},u) \mid < B_{\dot{\Delta}(\textbf{x},u)}
\end{equation}
where $B_{\dot{\Delta}(\textbf{x},u)}$ is a positive constant. 
\end{assumption}

\begin{theorem}[System stabiltiy]\label{theorem2}
The CLELC algorithm \eqref{eq_new_controllaw} is applied to the feedback linearizable system \eqref{eq_nonlinearsystem}. If the learning rate $\alpha$ is large enough as stated in Assumption \ref{ass_dotDelta}, i.e., $\alpha >  B_{\dot{\Delta}(\textbf{x},u)}$, the sliding surface $s$ approaches to zero for an arbitrary initial condition $s(0)$ within finite time $t_{h}$, which is estimated as
\begin{equation} \label{eq_th2}
t_{h} \leq \frac{ \mid s(0) \mid}{\alpha -  B_{\dot{\Delta}(\textbf{x},u) }}
\end{equation}
This purports that the sliding surface $s(t)$ approaches to zero for all $t > t_{h}$. 
\end{theorem}

\begin{proof}
The stability of an $n$th order feedback linearizable system is checked by following  Lyapunov function:
\begin{equation}\label{eq_lyapunov}
V = \frac{1}{2} s^{2}  
\end{equation}
The time derivative of the Lyapunov function is calculated as follows: 
\begin{equation}\label{eq_lyapunovrate}
\dot{V} =   s \dot{s} 
\end{equation}
The sliding surface rate $\dot{s}$ is obtained by taking \eqref{eq_ss2} into account. Then, it is inserted into \eqref{eq_lyapunovrate} as follows:
\begin{equation}\label{eq_lyapunovratee}
\dot{V} =  s ( \ddot{e}_{n} + \dot{u} _{b} )
\end{equation}

If the control law for a CLELC algorithm \eqref{eq_new_controllaw} is applied to the feedback linearizable system with uncertainties \eqref{eq_nonlinearsystem}, it is obtained
\begin{equation}\label{eq_xn}
\dot{x}_{n} = u_{b} + u_{f} + u_{n}  + \Delta (\textbf{x},u) 
\end{equation}
Then, the closed-loop error dynamics are obtained by inserting the feedforward control action \eqref{eq_ff}
\begin{equation}\label{eq_closedloop}
\dot{e}_{n} = - u_{b} - u_{n} - \Delta (\textbf{x},u) 
\end{equation}
Time derivative of \eqref{eq_closedloop} is taken and \eqref{eq_lyapunovratee} is rewritten as below:
\begin{equation}\label{eq_lyapunovrate2}
\dot{V} = s \big(  -\dot{u}_{n} - \dot{\Delta}(\textbf{x},u) \big) 
\end{equation}

If the time derivative of the neuro-fuzzy output \eqref{eq_dotun5} is inserted into the aforementioned equation, it is computed below:
\begin{equation}\label{eq_lyapunovrate3}
\dot{V} = s \big( - \alpha \textrm{sgn}\left( s \right) - \dot{\Delta}(\textbf{x},u) \big) 
\end{equation}
As stated in Assumption \ref{ass_dotDelta}, the rate of the uncertainty term is upper bounded
\begin{equation}\label{eq_lyapunovrate2}
\dot{V}  \leq  \mid s \mid  \big( -  \alpha + B_{\dot{\Delta}(\textbf{x},u)} \big)
\end{equation}
If the learning rate $\alpha$ is large enough, i.e., $\alpha >B_{\dot{\Delta}(\textbf{x},u)}$, as remarked in Theorem \ref{theorem2}, the time derivative of the Lyapunov function yields
\begin{equation}
 \dot{V}< 0
 \end{equation}
so that the sliding surface approaches to zero in a stable manner so that the closed-loop system is stable.

It is possible now to be shown that such a convergence takes place in finite time for an arbitrary initial condition. The sliding surface rate is satisfied by the sliding surface trajectory $s(t)$ is as below:
\begin{equation}\label{eq_finitetime1}
\dot{s}(t) =  - \alpha \textrm{sgn}\left( s(t) \right) - \dot{\Delta}(\textbf{x},u) 
\end{equation}
the solution of the sliding surface rate with an initial condition $s(0)$ satisfies for any $t \leq t_{h}$
\begin{equation}\label{eq_finitetime2}
s(t) - s(0) = \int^{t}_{0} \dot{s}(\xi) d\xi =  \int^{t}_{0} \big(- \alpha \textrm{sgn}\left( s(\xi) \right) - \dot{\Delta}(\textbf{x},u) \big) d\xi
\end{equation}
the solution takes zero value at time $t=t_{h}$, i.e., $s(t_{h})=0$, and the arbitrary initial condition is inserted into the equation above, i.e., $\textrm{sgn}\left( s(\xi) \right) = \textrm{sgn}\left( s(0) \right)$, therefore
\begin{eqnarray}\label{eq_finitetime3}
\underbrace{s(t_{h}) }_{0}-s(0) &=&   \int^{t_{h}}_{0} \big(- \alpha \textrm{sgn}\left( s(0) \right) - \dot{\Delta}(\textbf{x},u) \big) d \xi \nonumber \\
-s(0) &=&   - \alpha \textrm{sgn}\left( s(0) \right) t_{h}  - \int^{t_{h}}_{0}  \dot{\Delta}(\textbf{x},u) d \xi 
\end{eqnarray}
Both sides of the aforementioned equation are multiplied by $-\textrm{sgn}\left( s(0) \right)$, it is obtained as below:
\begin{eqnarray}\label{eq_finitetime4}
\mid s(0) \mid  &=&   \alpha t_{h}  + \textrm{sgn}\left( s(0) \right) \int^{t_{h}}_{0}  \dot{\Delta}(\textbf{x},u) d \xi \nonumber \\
 & \geq &   \alpha t_{h} - B_{\dot{\Delta}(\textbf{x},u)} t_{h} 
\end{eqnarray}
A finite time $t_{h}$ is found as
\begin{equation}\label{eq_finitetime5}
 t_{h} \leq \frac{ \mid s(0) \mid }{  \alpha - B_{\dot{\Delta}(\textbf{x},u)} }
\end{equation}
This purports that a sliding motion is maintained on $s(t)=0$ $\forall t > t_{h}$.
\end{proof}

\begin{remark}
As stated in Theorems  \ref{theorem1} and \ref{theorem2}, the learning rate $\alpha$ must be  large enough to guarantee the stability of the learning algorithm and the system stability; however, it is to be noted that a larger $\alpha$ may cause the chattering effect. Therefore, it is so crucial to determine the learning rate $\alpha$. In this study, the $\textrm{sgn}$ functions are replaced by the following equation to decrease the chattering effect $\textrm{sgn} \left( s \right) = \frac{ s}{\mid s \mid +\delta}$ where $\delta=0.05$. It is important to state that the selection of $\delta$ is very critical and must be carried out by taking the learning rate $\alpha$ into account to reduce the chattering effect: (i) if $\delta$ is small, it may not be enough to reduce the chattering effect, (ii) if $\delta$ is large, it may cause steady-state error. Therefore, an appropriate value is selected considering anticipated disturbance in this paper.
\end{remark}

\section{Simulation Results}\label{sec_simresults}

The following third-order numerical system is considered for the simulation study:
\begin{eqnarray}\label{eq_3rdsystems}
  \left[
  \dot{x}_{1},   \dot{x}_{2},   \dot{x}_{3}
  \right]^{T}   =\left[ x_{2}, x_{3}, a(\textbf{x})+b(\textbf{x})u +  \Delta (\textbf{x},u)  \right]^{T}
\end{eqnarray}
where $a(\textbf{x})=-2 x_{1} - x_{2} - \sin{(x_{3})} + e^{x_{1}}$ and $b(\textbf{x})=1$. 

The simulation study runs for a total of $20$ s with a sampling time of $0.01$ s. The number of membership functions for each input to the adaptive neuro-fuzzy structure is set to $3$. The system states are initialized as $\textbf{x}(0)=[1, -1, -10]^{T}$. The system is a third order system; therefore, the sliding surface is obtained as $s = \dot{e}_{3} + 3 \lambda e_{3} + 3 \lambda^{2} e_{2} + \lambda^{3} e_{1}$. Moreover, the slope of the sliding surface $\lambda$ is selected as $3$; therefore, the sliding surface is obtained as $s= \dot{e}_{3} + 9 e_{3} + 27 e_{2} + 27 e_{1}$ and the controller gain vector is obtained as $\textbf{k}=[k_{1}, k_{2}, k_{3}]= [27, 27, 9]$ as formulated in \eqref{eq_ks}. The reference signal $r$ is equal to zero, i.e., $\textbf{r}=0$. The disturbance is imposed on the system at $10$ s and formulated as $\Delta (\textbf{x},u) =5 \sin{(t)}$. As remarked in Theorem \ref{theorem2}, the learning rate $\alpha$ must be larger than the maximum value of the uncertainty rate, i.e., $\alpha >B_{\dot{\Delta}(\textbf{x},u)}$ to guarantee the system stability. Therefore, the learning rate $\alpha$ is set to $25$, i.e., $\alpha=25 > B_{\dot{\Delta}(\textbf{x},u)}=5$.

The performance of the CLELC algorithm is shown in Figs. \ref{fig_x1}-\ref{fig_x3}. At the beginning of the simulation between $t=0-10$ s, the control performances of the FLC and CLELC algorithms are compared while there is no uncertainty in the system. The CLELC results in smaller rise and settling times, and overshoot than the traditional FLC in the absence of uncertainties. It can be concluded that the CLELC algorithm improves the transient response performance of the system. In next step, a sinusoidal disturbance $\Delta (\textbf{x},u) =5 \sin{(t)}$ is imposed on the system between $t=10-20$ s to evaluate the robustness of the CLELC algorithm against uncertainty. The traditional FLC is not robust against uncertainties after the disturbance is imposed on the system. The CLELC algorithm learns system behavior online by adjusting its parameters and controls the system without any steady-state error. Thanks to the CLELC scheme, the developed control algorithm results in robust control performance in the presence of uncertainties. 

The sliding surface of the sliding mode learning algorithm must converge to zero in finite time, which is proven in Theorem \ref{theorem2}. Taking \eqref{eq_finitetime5} into account, the finite time is found as $t_{h} \leq 4.5$ seconds, where $s(0)=-90$, $\alpha=25$ and $B_{\dot{\Delta}(\textbf{x},u)}=5$. As can be seen from Figs. \ref{fig_x1}-\ref{fig_x3}, the states converges to zero less than $4.5$ seconds, which confirms the formulation in \eqref{eq_finitetime5}.

The control signals for the CLELC algorithm are shown in Fig. \ref{fig_controlsignals}. The intelligent controller takes the overall control of the control signal in a very short duration; therefore, the feedback control input $u_{b}$ approaches to zero. This implies that the CLELC algorithm learns the system dynamics. Note that the feed-forward control input is set to zero throughout the simulation study since the reference is equal to zero. 

\begin{figure}[t!]
\centering
\subfigure[Responses of state $x_{1}$.]{
\includegraphics[width=0.46\columnwidth]{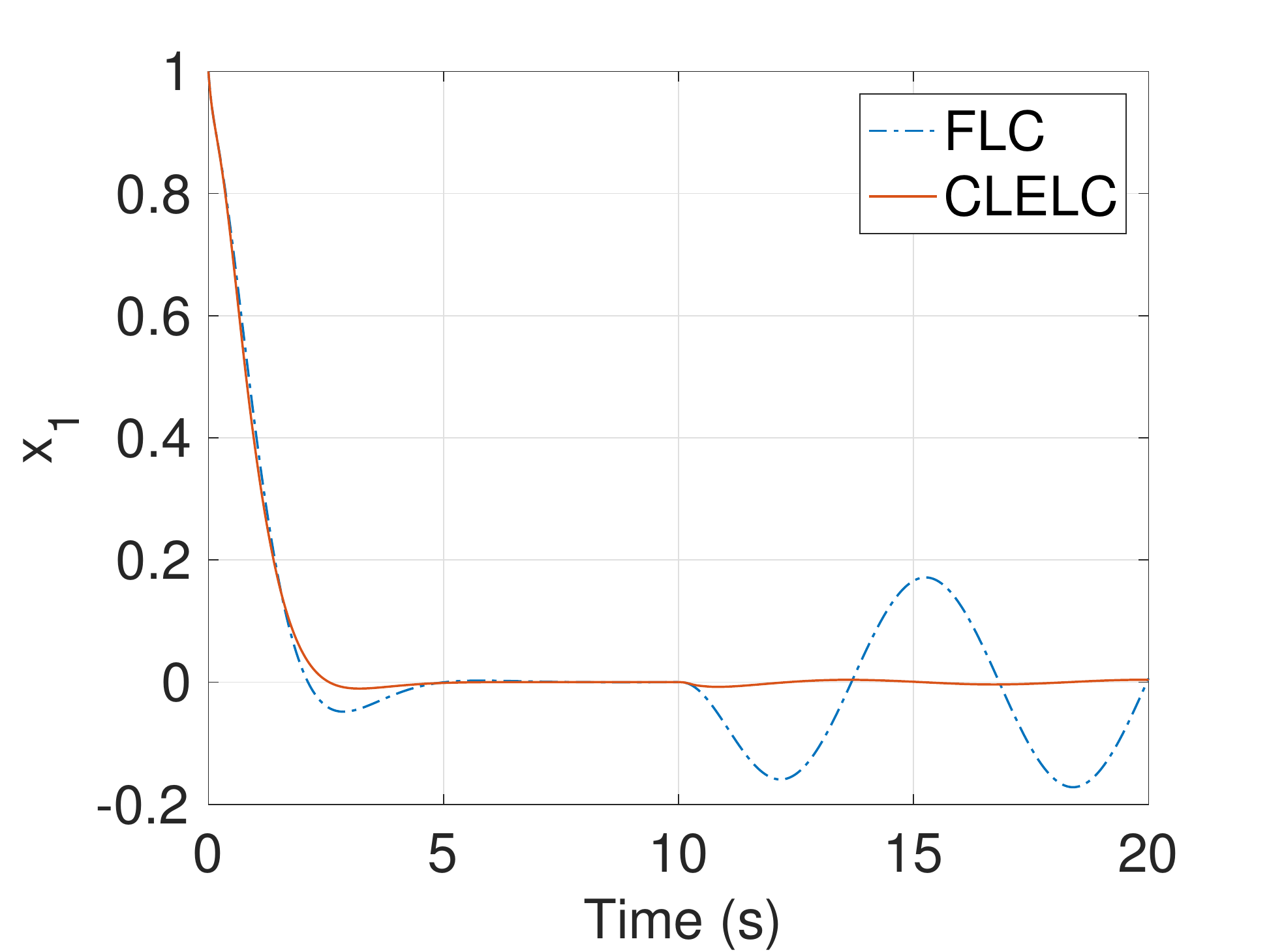}
\label{fig_x1}
}
\subfigure[Responses of state $x_{2}$. ]{
\includegraphics[width=0.46\columnwidth]{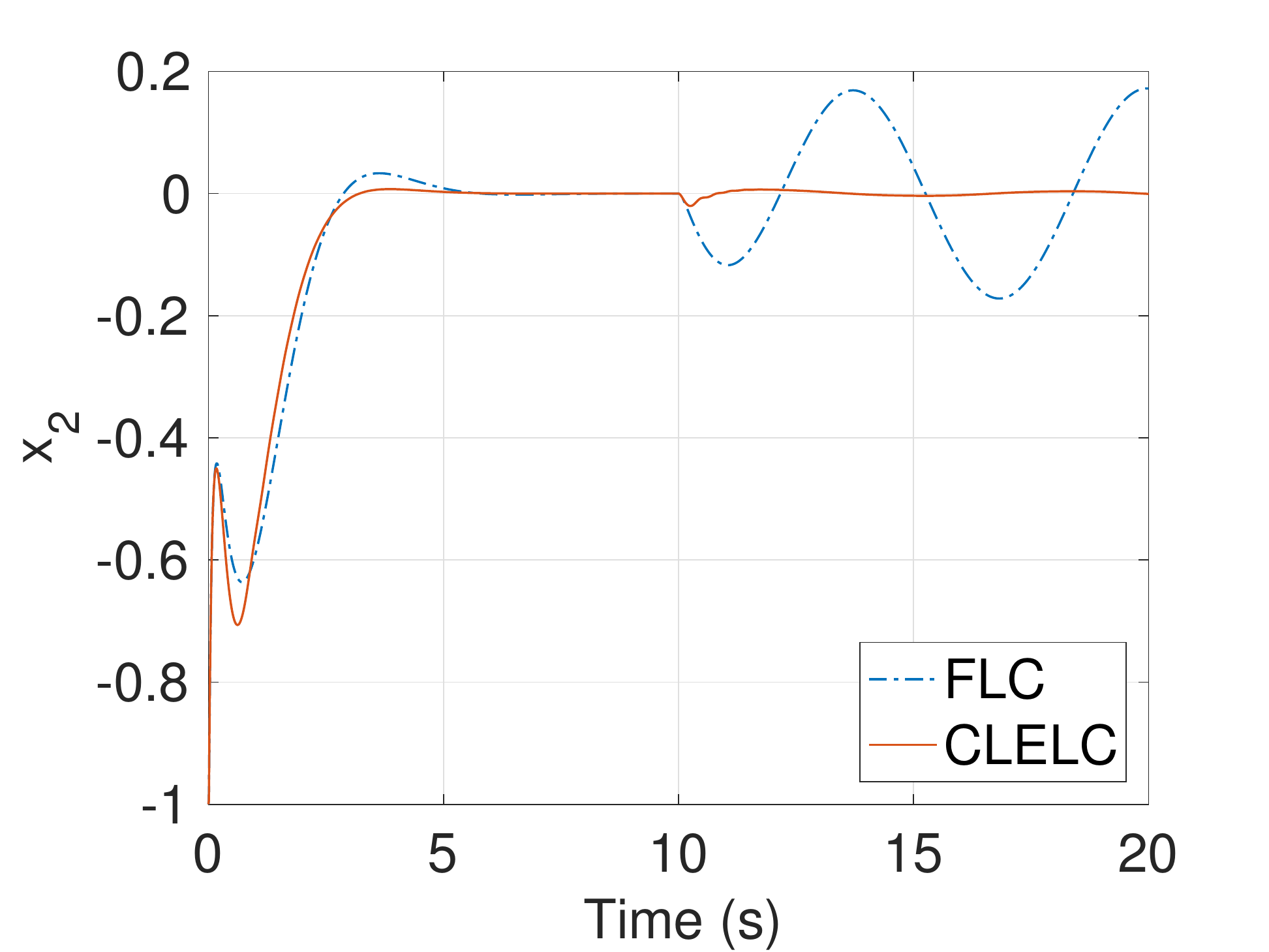}
\label{fig_x2}
}
\subfigure[Responses of state $x_{3}$.]{
\includegraphics[width=0.46\columnwidth]{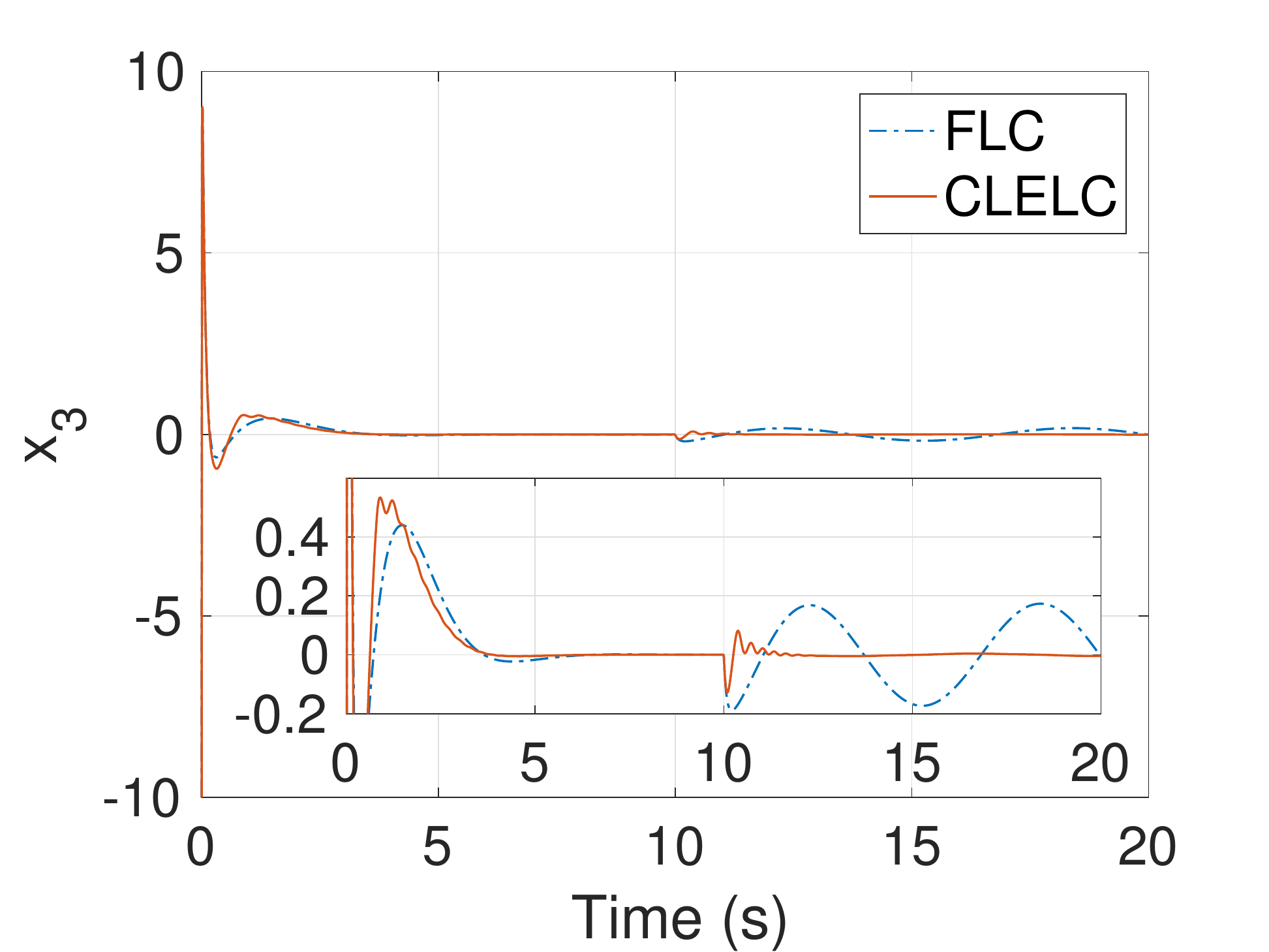}
\label{fig_x3}
}
\subfigure[Control signals.]{
\includegraphics[width=0.46\columnwidth]{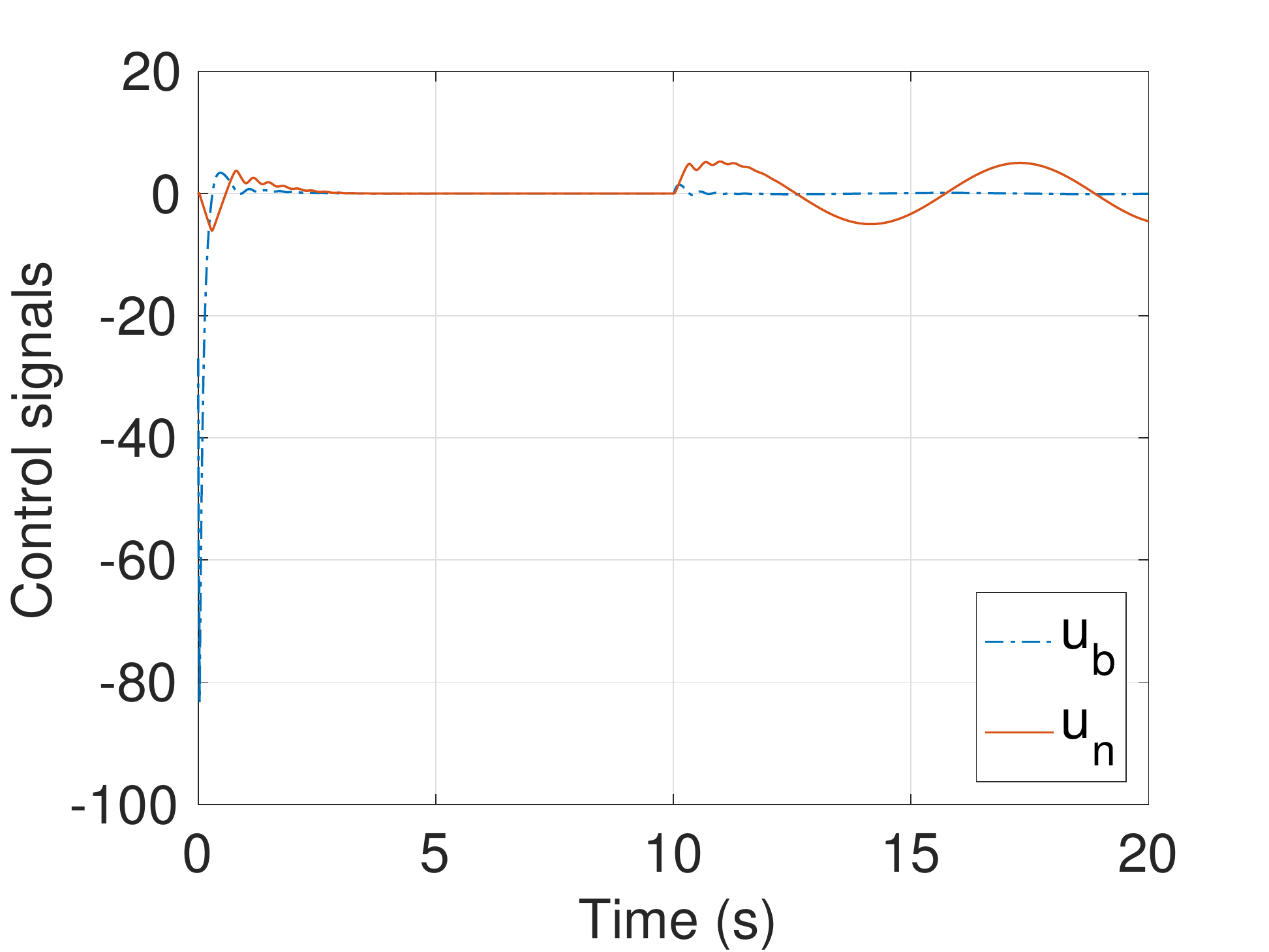}
\label{fig_controlsignals}
}
\caption[Optional caption for list of figures]{Simulation results.}
\label{fig_sim}
\end{figure}

\section{Experimental Validation}\label{sec_expval}

\subsection{Mobile Robot}\label{sec_mobilerobot}

\begin{figure}[b!]
  \centering
  \includegraphics[width=0.6\columnwidth]{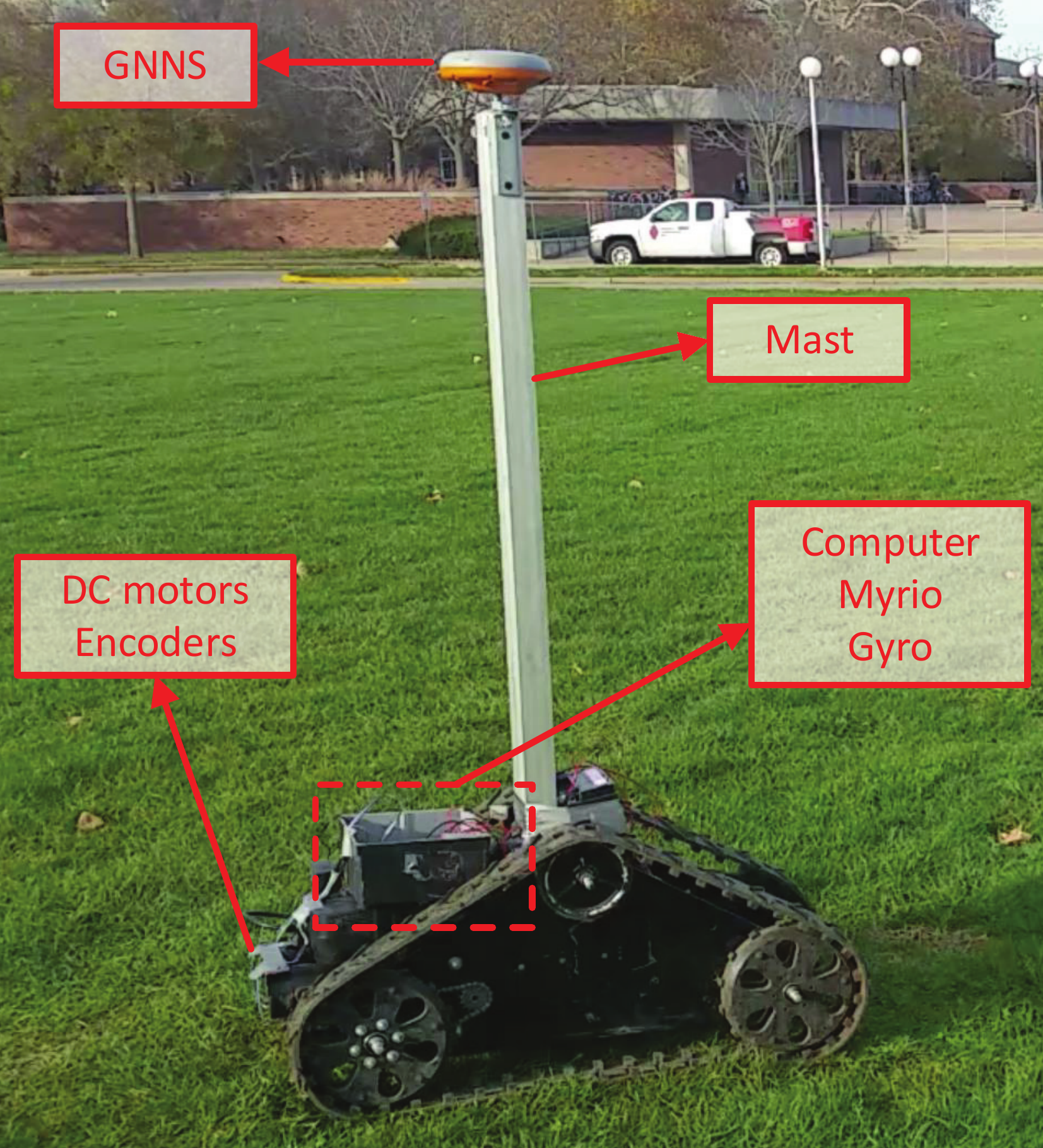}
  \caption{Mobile robot.}\label{fig_robot_labeled}
\end{figure}
The mobile robot shown in Fig. \ref{fig_robot_labeled} has been equipped with sensors. Plant imaging and phenotyping sensors, as well as a Real-Time Kinematic (RTK) differential Global Navigation Satellite System (GNSS) are mounted on the mast on the robot. A Septentrio Altus APS-NR3 GNSS (Septentrio Satellite Navigation NV, Belgium) has been employed to acquire the position of the mobile robot with an accuracy of $0.04$ m at a $5$-Hz sampling rate. More information about the actuators and sensors mounted on the mobile robot can be found in \cite{erkanjfr}.

The CLELC algorithm is implemented in LabVIEW$^{TM}$ (v2015, National Instrument Corporation, USA) and executed in real time on the onboard computer. It gathers the GNSS data, linear and angular velocities, and regulates the speed of the tracks by sending signals to the motion controller. The control input generated by the CLELC algorithm is sent to the motion controller via NI Myrio (National Instrument Corporation, USA), and linear and angular velocities measurements are received by the same way. Position measurements are received via Bluetooth. The sampling frequency of the CLELC algorithm is set to $5$-Hz.

\subsection{Feedback Linearization Control of a Mobile Robot}

The schematic diagram of the mobile robot is illustrated in Fig. \ref{fig_mobilerobot}. The mobile robot is a unicycle type robot and formulated as \cite{Orlolo2002, 7468507, 7562522}:
\begin{eqnarray}\label{eq_mobilerobotmodel}
 \left[
 \dot{p_{x}} ,  \dot{p_{y}} , 	\dot{\theta}
  \right]^{T}   =\left[ v \cos{(\theta)}, v \sin{(\theta)}, \omega 
  \right]^{T}
\end{eqnarray}
where $p_{x}$ and $p_{y}$ are respectively the position of the mobile robot in x- and y-axes, $\theta$ is the heading angle, $v$ is the linear velocity, and $\omega$ is the angular velocity. 

\begin{figure}[h!]
  \centering
  \includegraphics[width=0.6\columnwidth]{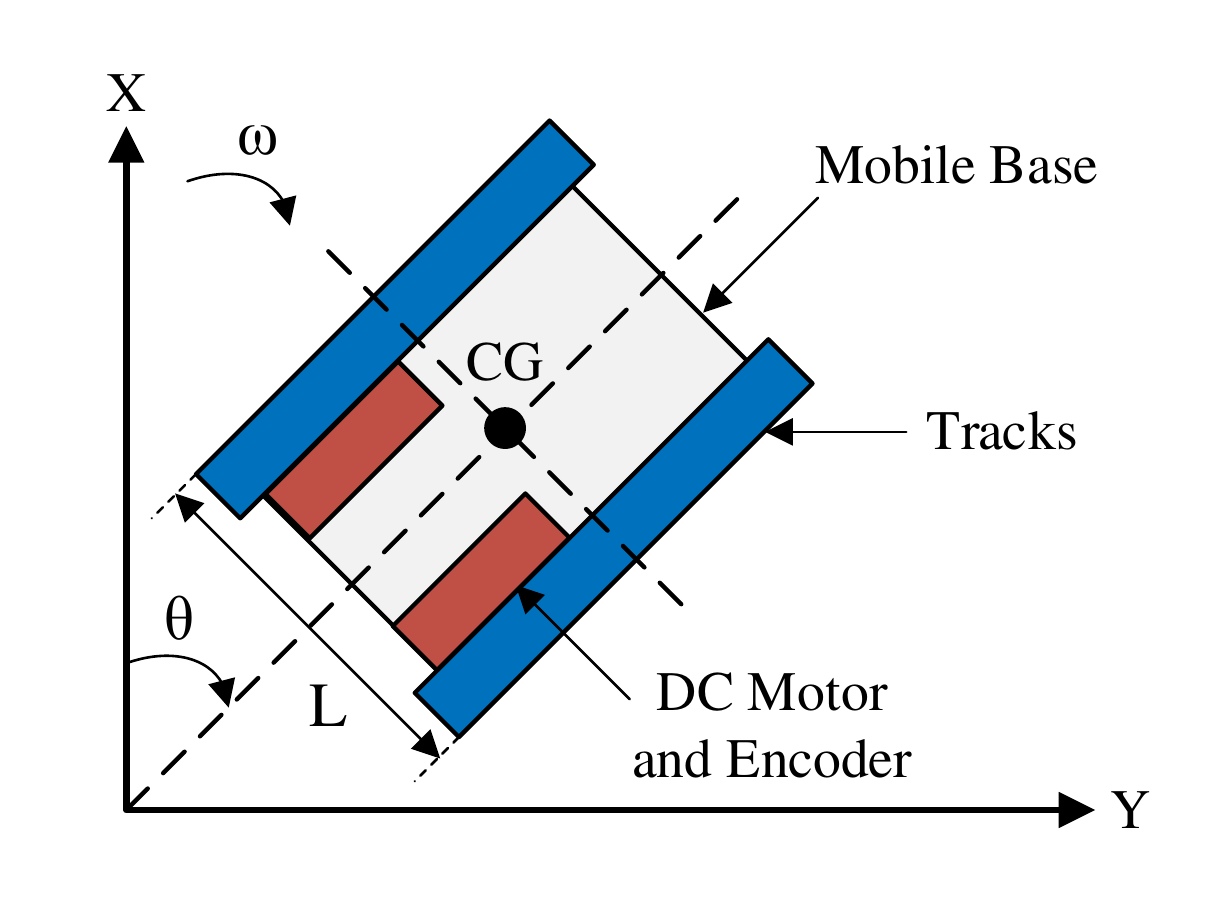}
  \caption{Schematic illustration of the mobile robot.}\label{fig_mobilerobot}
\end{figure}

The feedback linearization procedure is illustrated for the unicycle model \eqref{eq_mobilerobotmodel} in literature \cite{Orlolo2002}. The linearizing states are defined as $x_{1}=p_{x}$, $x_{2}=p_{y}$. If $x_{1}$ and  $x_{2}$ are differentiated taking \eqref{eq_mobilerobotmodel} into account, 
\begin{eqnarray}
\left[
  \begin{array}{c}
   \ddot{x}_{1} \\
   \ddot{x}_{2} 
  \end{array}
  \right]  = \dot{v} \left[
  \begin{array}{c}
  \cos{(\theta)}\\
   \sin{(\theta)}
  \end{array}
  \right] +  v \dot{\theta} \left[
  \begin{array}{r}
  -\sin{(\theta)}\\
   \cos{(\theta)}
  \end{array}
  \right] 
\end{eqnarray}
Let us define linear acceleration as $\xi=\dot{v}$ and consider $\dot{\theta}=\omega$. The aforementioned equation is re-arranged as follows:
\begin{eqnarray}
\left[
  \begin{array}{c}
   \ddot{x}_{1} \\
   \ddot{x}_{2} 
  \end{array}
  \right]  =\left[
  \begin{array}{l r}
  \cos{(\theta)} &  -v \sin{(\theta)}\\
   \sin{(\theta)} &  v \cos{(\theta)}
  \end{array}
  \right] \left[
  \begin{array}{c}
  \xi\\
   \omega
  \end{array}
  \right] 
\end{eqnarray}
If the linear velocity is not equal to zero, i.e., $v \neq 0$, the matrix is non-singular. Thus, $\xi$ and $\omega$ are obtained as follows: 
\begin{eqnarray}
\left[
  \begin{array}{c}
  \xi\\
   \omega
  \end{array}
  \right]   =\left[
  \begin{array}{l r}
  \cos{(\theta)} &  -v \sin{(\theta)}\\
   \sin{(\theta)} &  v \cos{(\theta)}
  \end{array}
  \right]^{-1} \left[
  \begin{array}{c}
   \ddot{x}_{1} \\
   \ddot{x}_{2} 
  \end{array}
  \right] 
\end{eqnarray}
The system is linearized and described by two chains of integrator
  $\left[
  \ddot{x}_{1},   \ddot{x}_{2}
  \right]^{T}   =\left[  u_{1} ,   u_{2}
  \right]^{T}$
so that the feedback linearization law is obtained as follows:
\begin{equation}\label{eq_transformation}
\dot{v} =u_{1} \cos{\theta} + u_{2} \sin{\theta}, \quad \textrm{and} \quad
\omega =\frac{u_{2} \cos{\theta} - u_{1} \sin{\theta}}{v}
\end{equation}
The new states are defined as $\textbf{x}=[x_{1}, x_{2}, x_{3}, x_{4}]^{T}=[p_{x}, p_{y}, \dot{p}_{x}, \dot{p}_{y}]^{T}$ and the new state-space equations are written as follows:
\begin{eqnarray}
\begin{matrix}
\dot{x}_{1} & = & x_{3} \\
  \dot{x}_{3} & = & u_{1} \\
\end{matrix} \qquad
\text{and}  \qquad
\begin{matrix}
\dot{x}_{2} & = & x_{4} \\
  \dot{x}_{4} & = & u_{2} \\
\end{matrix}
\end{eqnarray}

It is assumed that the mobile robot tracks a smooth trajectory $(r_{1}(t), r_{2}(t))$ and the desired linear velocity  along this trajectory never goes to zero. The control laws are written as follows:
\begin{eqnarray}\label{eq_flc_mobilerobot}
\left[
  \begin{array}{c}
u_{1}\\
u_{2}
  \end{array}
  \right]   =\left[
  \begin{array}{c}
  \ddot{r}_{1} + k_{1} ({r}_{1} - x_{1}) + k_{3} ({r}_{3} - x_{3}) \\
  \ddot{r}_{2} + k_{2} ({r}_{2} - x_{2}) + k_{4} ({r}_{4} - x_{4}) 
  \end{array}
  \right]
\end{eqnarray}
These generated signals are fed to the transformation matrix in \eqref{eq_transformation} to obtain the actual inputs, i.e., the desired linear and angular velocities, for a mobile robot.


\subsection{Implementation of the CLELC algorithm}
The CLELC laws for a mobile robot are written by utilizing \eqref{eq_new_controllaw} and \eqref{eq_flc_mobilerobot} as follows:
\begin{eqnarray}\label{eq_impcontrol}
\left[
  \begin{array}{c}
u_{1}\\
u_{2}
  \end{array}
  \right]   =\left[
  \begin{array}{c}
  \ddot{r}_{1} + k_{1} ({r}_{1} - x_{1}) + k_{3} ({r}_{3} - x_{3}) + u_{n1} \\
  \ddot{r}_{2} + k_{2} ({r}_{2} - x_{2}) + k_{4} ({r}_{4} - x_{4}) + u_{n2}
  \end{array}
  \right]
\end{eqnarray}
The number of membership functions for each input to each adaptive neuro-fuzzy structure is set to $3$ and the learning rate for the sliding mode learning algorithm $\alpha$ is set to $5$. 

After feedback linearization is applied, the mobile robot dynamics consists of two single-input-single-output systems which are second-order systems. Therefore, the sliding surfaces are obtained as $s _{1}= \dot{e}_{3} + 2 \lambda e_{3} + \lambda^{2} e_{1} $ and $s _{2}= \dot{e}_{4} + 2 \lambda e_{4} + \lambda^{2} e_{2} $. The slope of the sliding surface $\lambda$ is set to $0.3$; therefore,  the controller gain vector is obtained as $\textbf{k}=[k_{1}, k_{2},k_{3},k_{4}]=[0.09, 0.09, 0.6, 0.6]$ as formulated in \eqref{eq_ks}. 

The heading angle reference is derived from \eqref{eq_mobilerobotmodel} as follows:
\begin{equation}\label{eq_ref_ha}
\theta_{r} = \arctan2{(r_{4}, r_{3})} +k \pi \quad k=0,1
\end{equation}

To increase the tracking accuracy, feedforward actions are derived in a time-interval $t \in [0, T]$ by using the system model \eqref{eq_mobilerobotmodel} and the desired linear velocities in x- and y-axes $(r_{3}(t), r_{4}(t))$ as follows \cite{Kayacan2018}:
\begin{equation}
  v_{d} =\pm \sqrt{(r_{3})^{2} + (r_{4})^{2}}  \quad \textrm{and} \quad
  \omega_{d} = \frac{\dot{r_{4}} r_{3} - \dot{r_{3}} r_{4}}{(r_{3})^{2} + (r_{4})^{2}}
\end{equation}

\subsection{Experimental Results}\label{sec_expresults}

The start and end points of the reference trajectory and the initial point of the mobile robot are shown in Fig. \ref{fig_tra}. The traditional FLC and the CLELC can reach to the reference trajectory and stay on-track. The Euclidean errors for the FLC and CLELC are shown in Fig. \ref{fig_error}. The mean values of Euclidean errors for the mobile robot controlled by the FLC and CLELC after staying on-track are respectively $0.12$ $m$ and $0.03$ $m$ approximately. This purports that the CLELC results in more accurate trajectory tracking performance than the FLC, which demonstrates the capability of the CLELC algorithm. In the previous works, a nonlinear model predictive controller has been designed and implemented on the same mobile robot in \cite{Young2018}. Trajectory tracking error has been around $0.05$ $m$. Moreover, the kinematic model was augmented with two traction parameters and a nonlinear moving horizon estimator was used to estimate these traction parameters. Then, a nonlinear model predictive controller based on the adaptive kinematic model has been designed and implemented on the same mobile robot in \cite {erkanjfr}. Trajectory tracking error has been around $0.04$ $m$. Consequently, the CLELC algorithm gives more accurate trajectory tracking performance than previous implemented controllers on the same mobile robot.  

The reference and measured heading angles are shown in Fig. \ref{fig_ha}. The reference heading angle is calculated from the reference trajectory as formulated in \eqref{eq_ref_ha}. The mobile robot controlled by the CLELC is capable of tracking the heading angle reference accurately. The same consequence can be obtained from the heading angle error as shown in Fig. \ref{fig_haerror}. 

The feedback and intelligent control signals fed to the feedback linearization law \eqref{eq_transformation} are respectively  shown for $u_{1}$ and $u_{2}$ in Figs. \ref{fig_u1}-\ref{fig_u2}. The outputs of the feedback controllers converge to zero while the intelligent controllers learn the mobile robot behavior online. The intelligent controllers take over the control, thus becoming the leading controllers.

The output of the feedback linearization law \eqref{eq_transformation} results in the desired linear and angular velocities fed to the mobile robot. The linear velocity reference and measurements of the mobile robot are shown in Fig. \ref{fig_speed}. The linear velocity reference of the mobile robot is changing around a constant speed because the linear velocity throughout the reference trajectory is constant. The angular velocity reference and measurements are shown in Fig. \ref{fig_yawrate}. The minimum and maximum angular velocities are limited by $-0.1$ and $0.1$ rad/s due to the capability of the mobile robot. The performance of the low-level controller is sufficient to track the reference signals despite the high noisy measurements.

\begin{figure}[t!]
\centering
\subfigure[Reference trajectory.]{
\includegraphics[width=0.46\columnwidth]{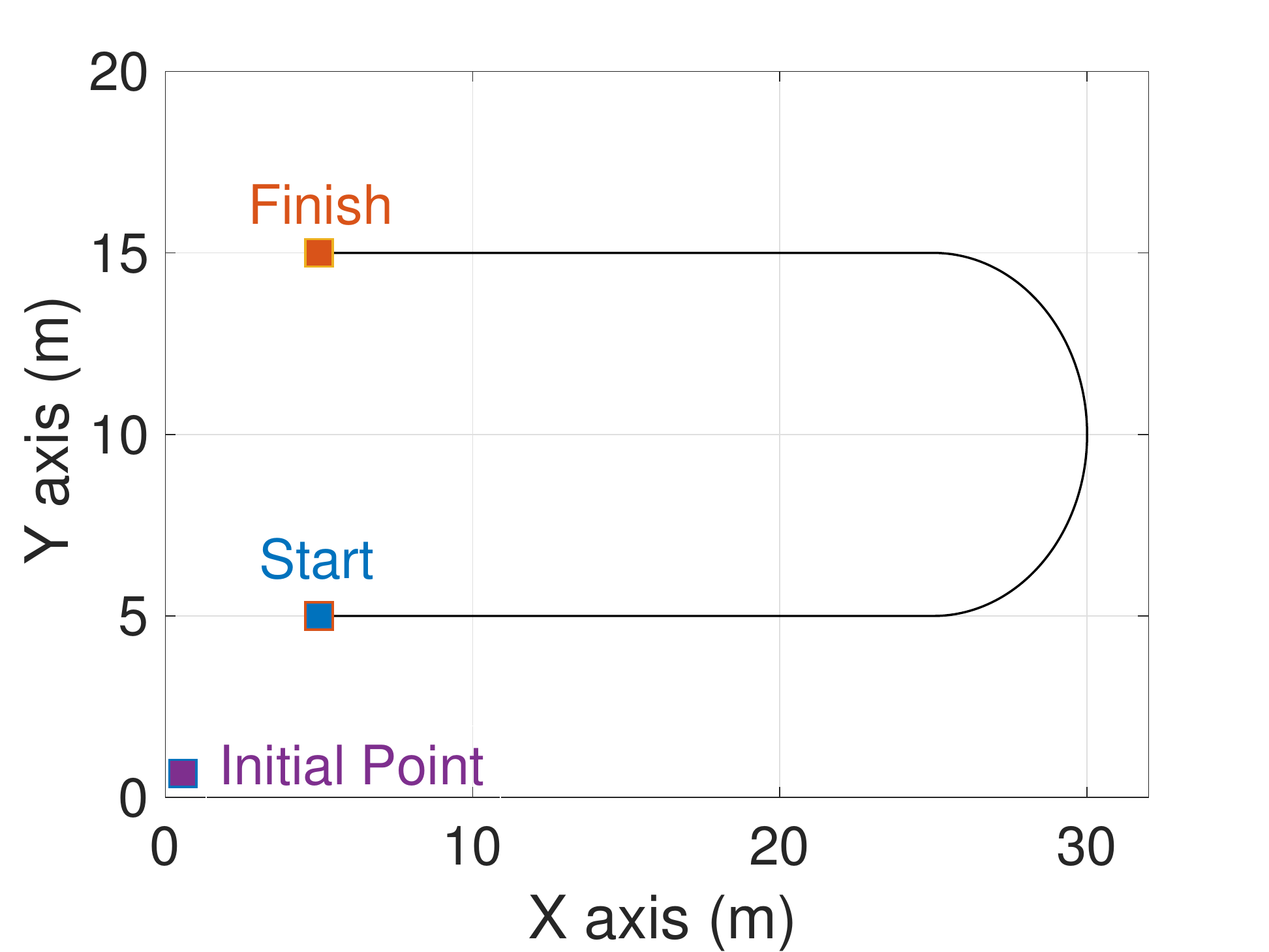}
\label{fig_tra}
}
\subfigure[Euclidean error.]{
\includegraphics[width=0.46\columnwidth]{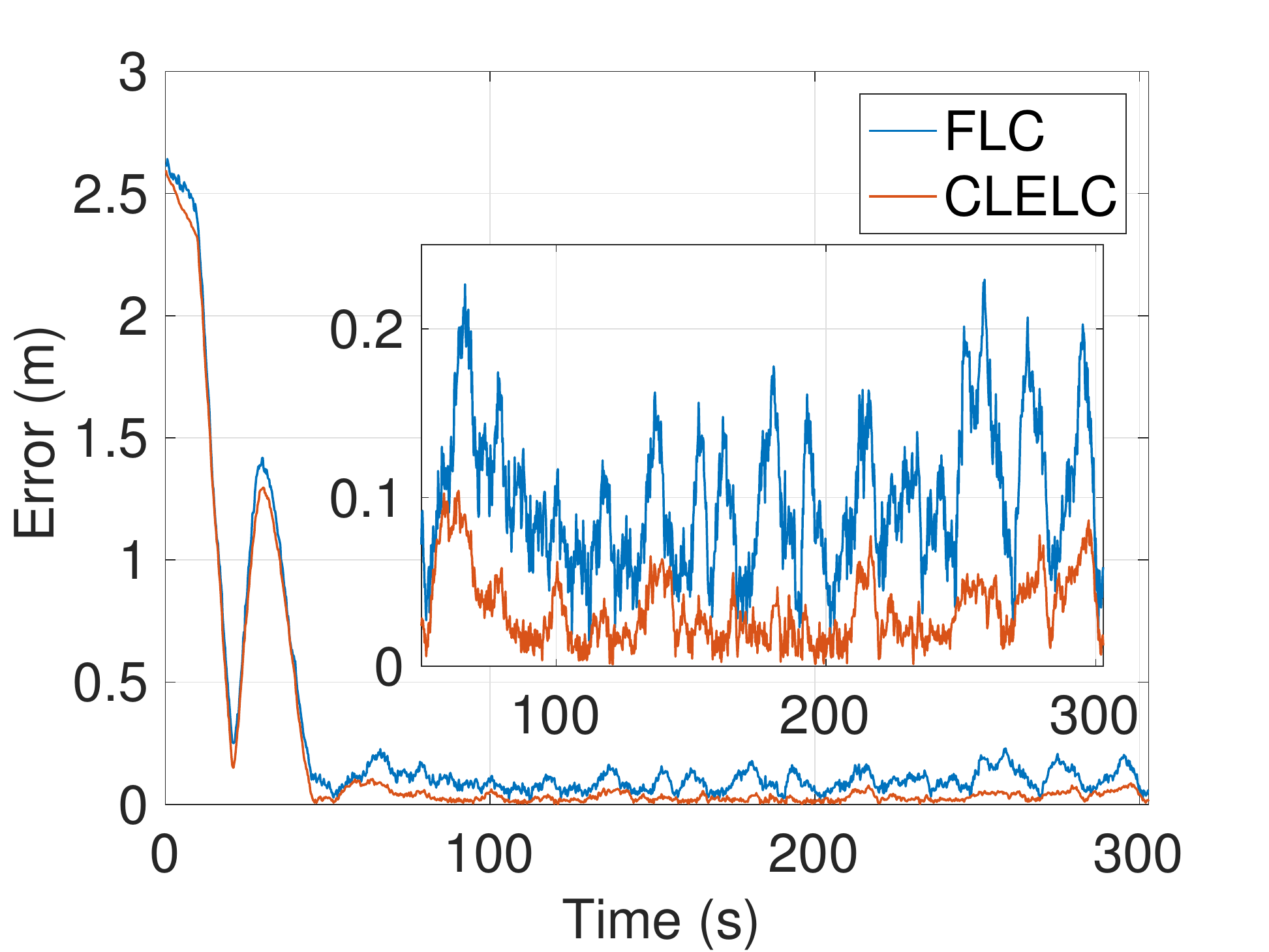}
\label{fig_error}
}
\subfigure[Heading angle reference and measurements.]{
\includegraphics[width=0.46\columnwidth]{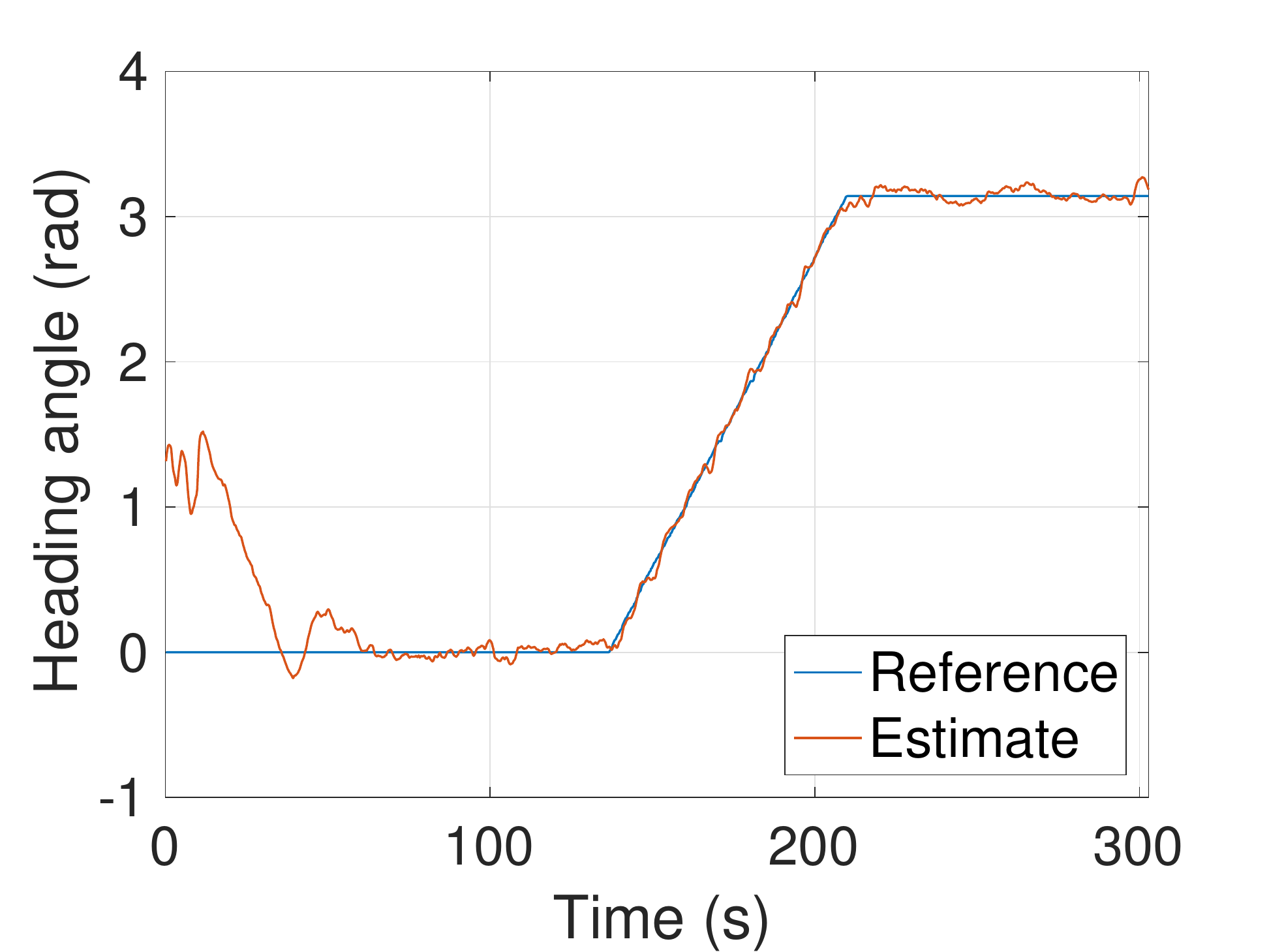}
\label{fig_ha}
}
\subfigure[Heading angle error.]{
\includegraphics[width=0.46\columnwidth]{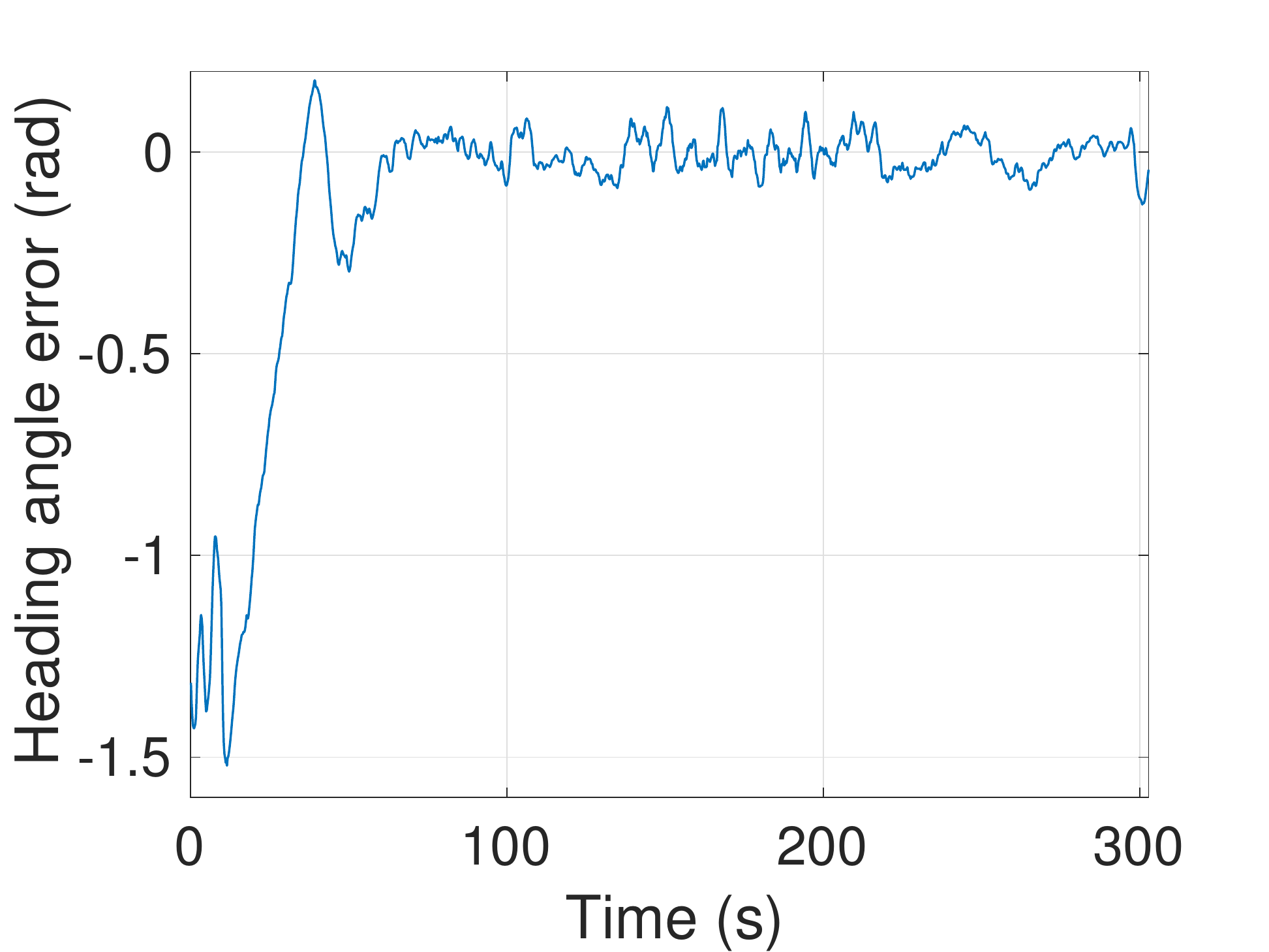}
\label{fig_haerror}
}
\subfigure[Feedback and intelligent control signals for $u_{1}$.]{
\includegraphics[width=0.46\columnwidth]{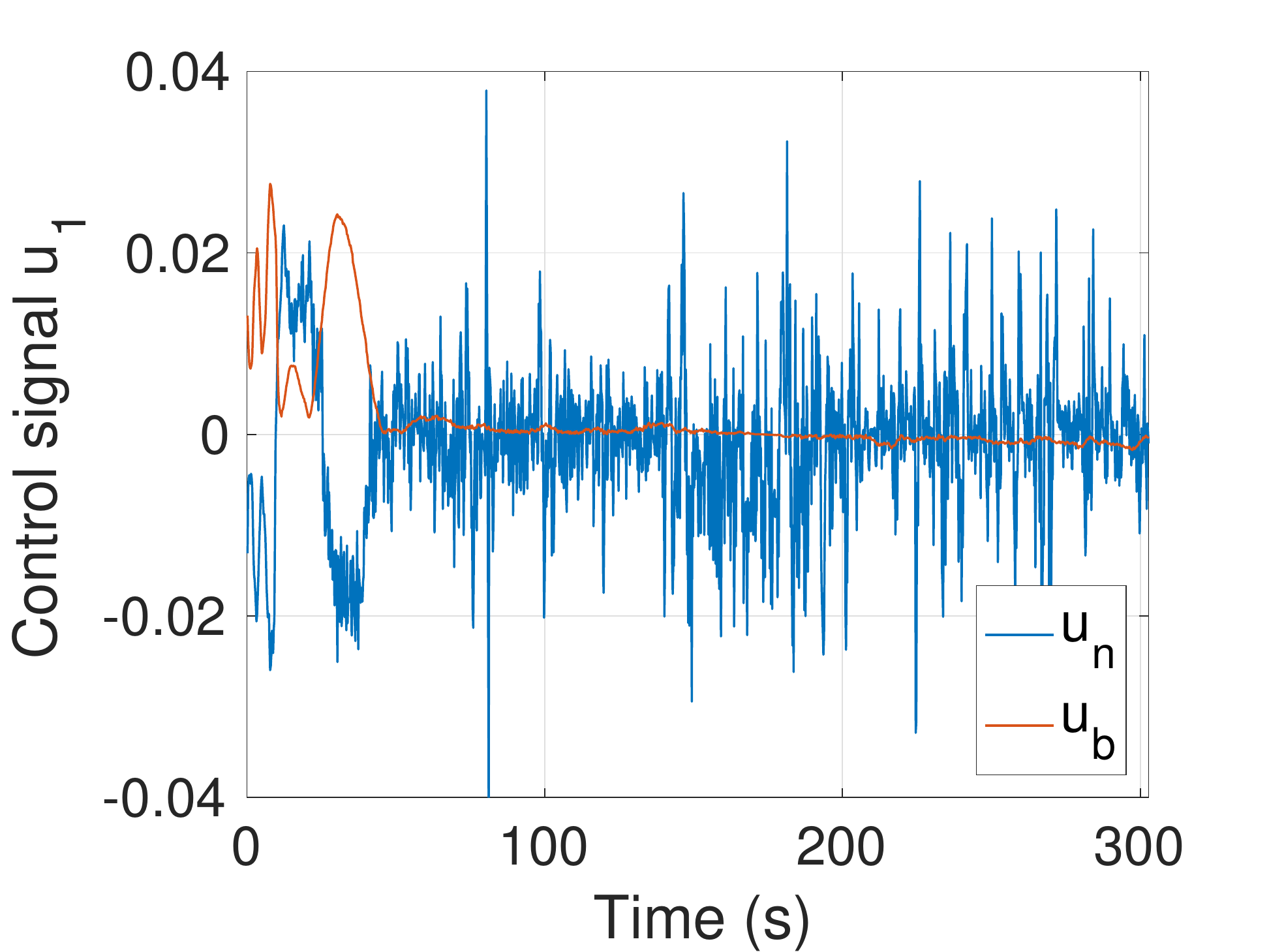}
\label{fig_u1}
}
\subfigure[Feedback and intelligent control signals for $u_{2}$.]{
\includegraphics[width=0.46\columnwidth]{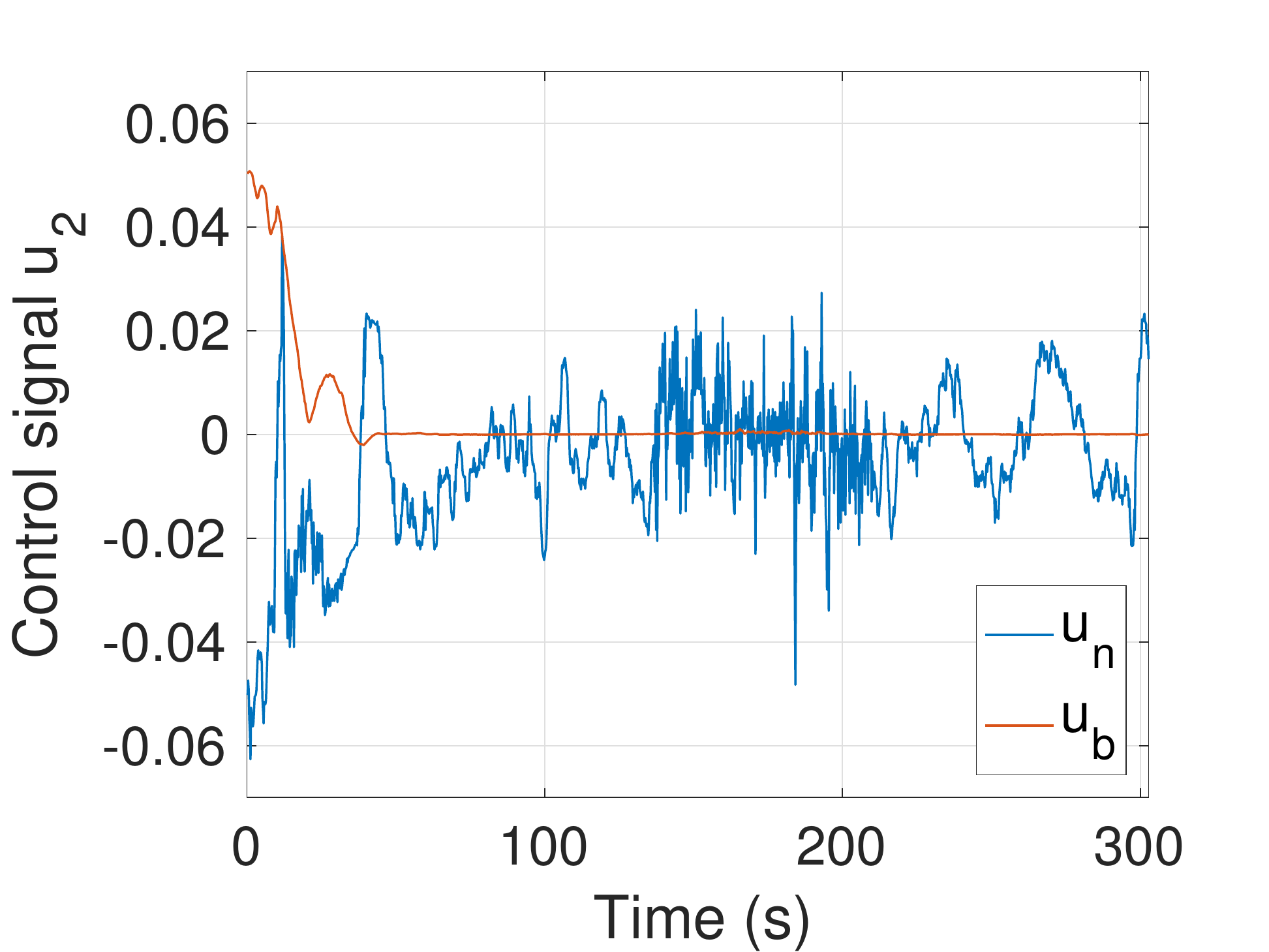}
\label{fig_u2}
}
\subfigure[Linear velocity reference and measurements.]{
\includegraphics[width=0.46\columnwidth]{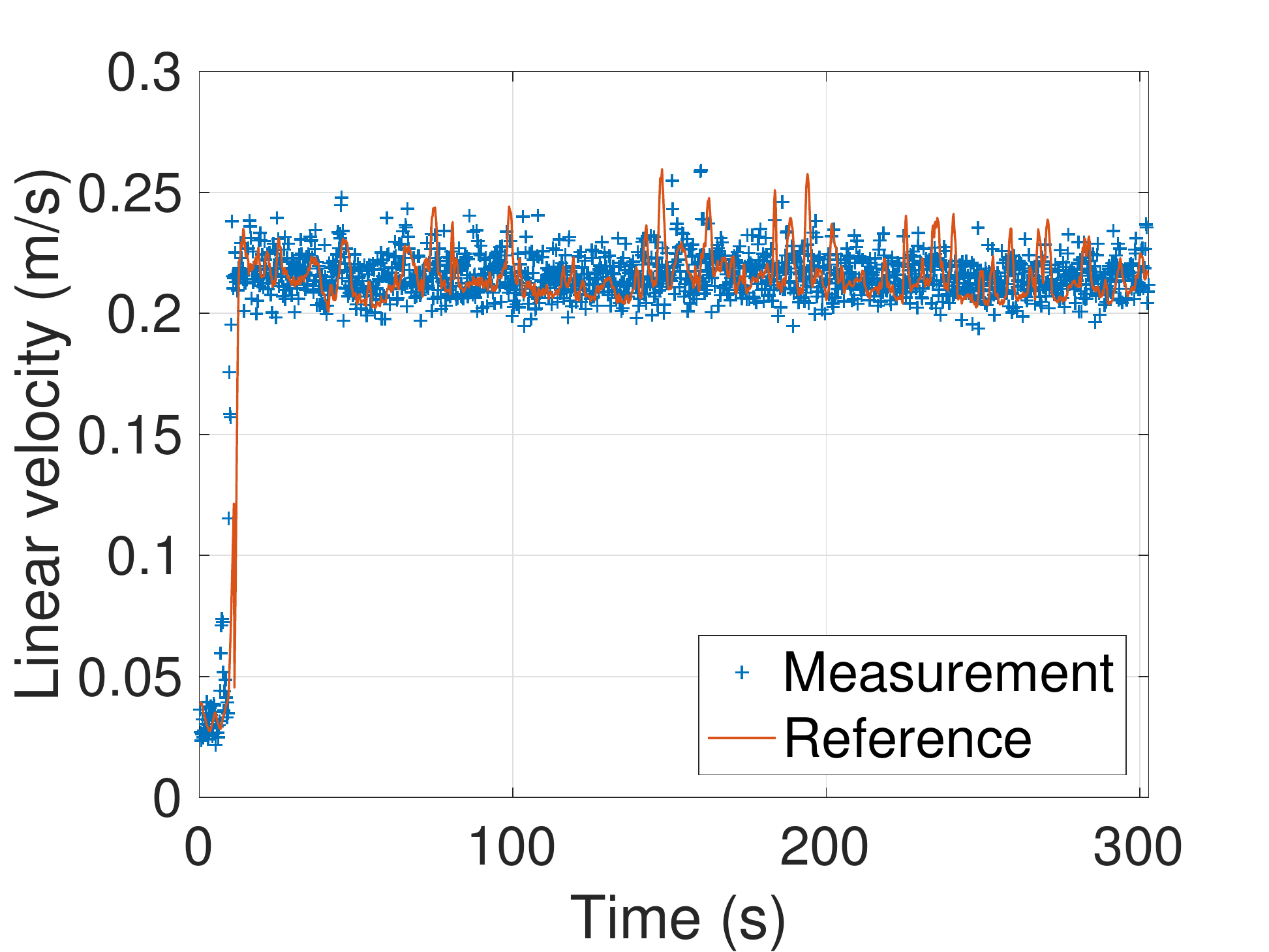}
\label{fig_speed}
}
\subfigure[Angular velocity reference and measurements.]{
\includegraphics[width=0.46\columnwidth]{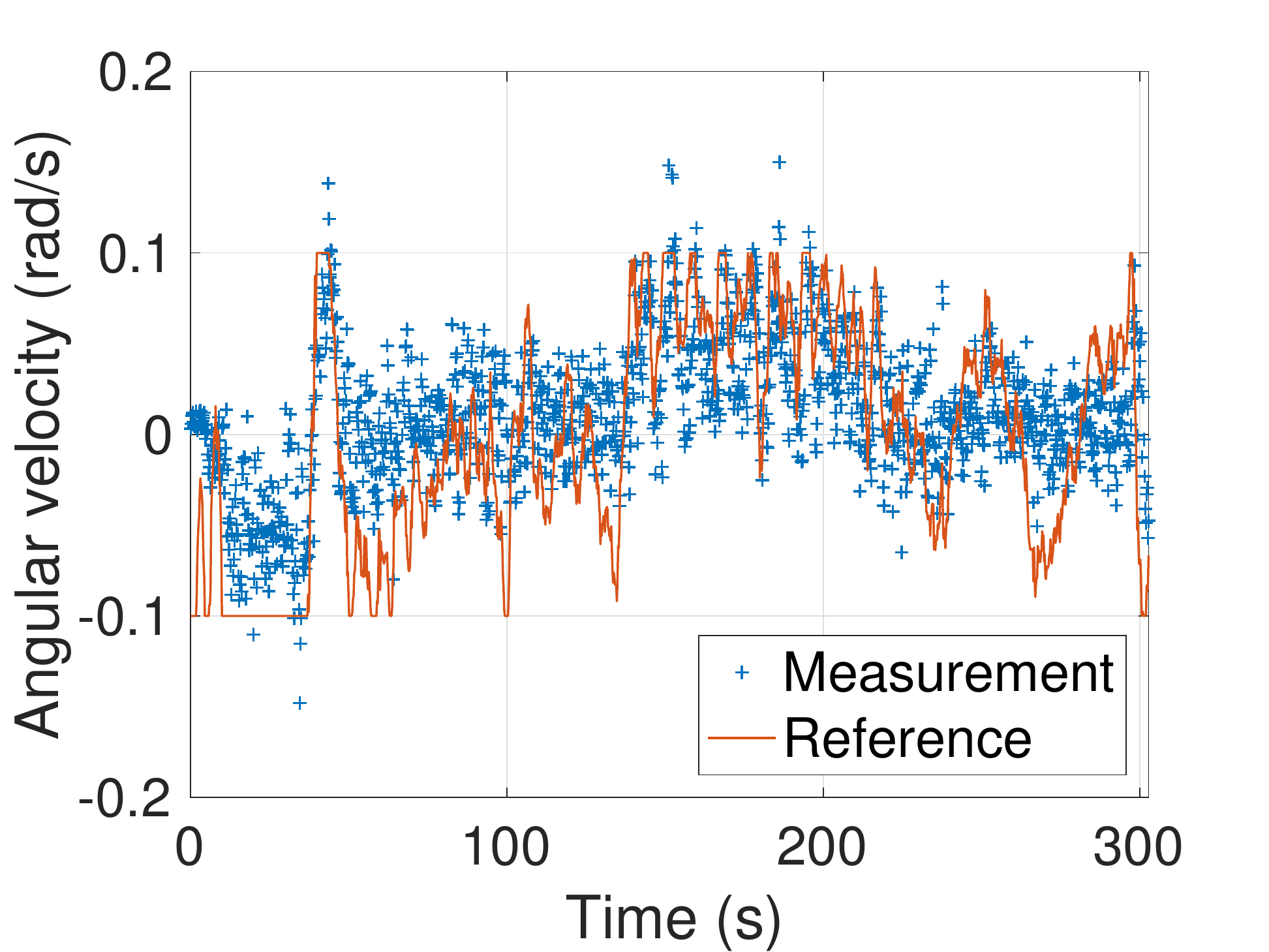}
\label{fig_yawrate}
}
\caption[Optional caption for list of figures]{ Experimental results on the mobile robot.}
\label{fig_exp}
\end{figure}

\section{Conclusions}\label{sec_conc}

In this study, a novel CLELC algorithm has been developed for the control of feedback linearizable systems with uncertainties and experimentally validated on a mobile robot. The simulation results show that the CLELC algorithm has exhibited better control performance in terms of smaller rise time, settling time and overshoot in the absence of uncertainties as compared to the traditional FLC. Moreover, it can drive the error to zero in the presence of uncertainties. The experimental results on a tracked mobile robot show that the CLELC algorithm ensures more accurate trajectory tracking performance when compared to the FLC. The mean values of the Euclidean errors for the FLC and CLELC after staying on-track are respectively $0.12$ m and $0.03$ m approximately. 

As a future study, the proposed algorithm can be developed for uncertain multi-input-multi-output nonlinear systems and the constraints on the error and reference rates in the stability analysis can be lightened.

\bibliography{ieee}
\bibliographystyle{IEEEtran}

\begin{IEEEbiography}[{\includegraphics[width=1in,height=1.25in,clip,keepaspectratio]{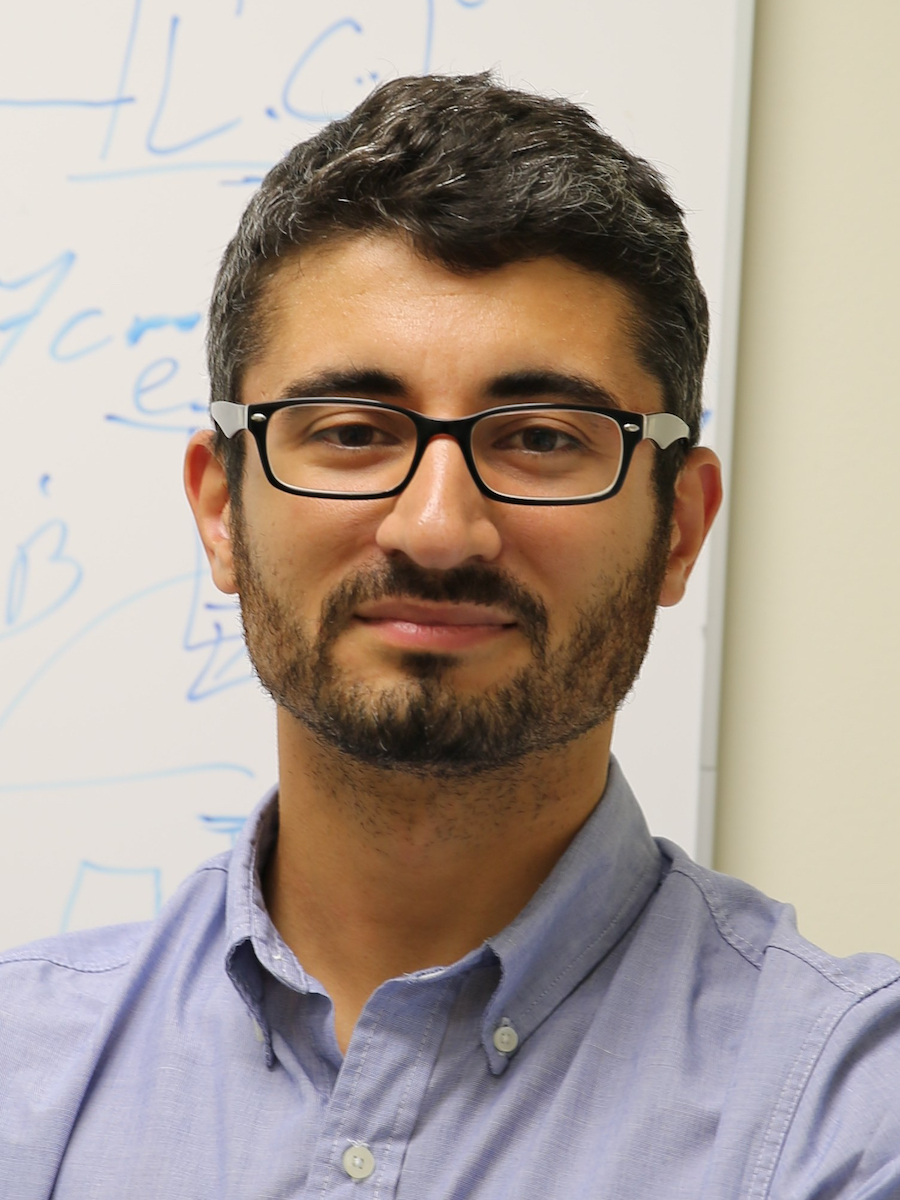}}]{Erkan Kayacan} (S\textquoteright 12-M\textquoteright 16-SM\textquoteright 19) received the B.Sc. degree in mechanical engineering and the M.Sc. degree in system dynamics and control from Istanbul Technical University, Istanbul, Turkey in 2008 and 2010, respectively. He received the Ph.D. degree in Mechatronics, Biostatistics and Sensors from University of Leuven (KU Leuven), Leuven, Belgium in 2014. 

He is currently a Lecturer with the School of Mechanical and Mining Engineering at the University of Queensland (UQ), Australia. Prior to UQ, he was a Postdoctoral Researcher with Delft University of Technology, University of Illinois at Urbana-Champaign, and Massachusetts Institute of Technology. His research interests include real-time optimization-based control and estimation methods, nonlinear control theory, learning algorithms and machine learning with a heavy emphasis on applications to autonomous systems and field robotics. 

Dr. Kayacan is a recipient of the Best Systems Paper Award at Robotics: Science and Systems (RSS) in 2018.

\end{IEEEbiography}

\clearpage

\end{document}